\documentclass[twoside,11pt]{article}
\usepackage{hyperref}

\usepackage{jair, theapa, rawfonts}
\usepackage{float}
\usepackage{graphicx}
\usepackage[caption=false]{subfig}
\usepackage{lmodern}
\usepackage{amsthm}
\usepackage{amsmath,amssymb,amsfonts}
\usepackage{booktabs} 
\usepackage{threeparttable}
\usepackage{tabularx}
\usepackage[section]{placeins}

\newtheorem{theorem}{Theorem}
\newtheorem{proposition}[theorem]{Proposition}%
\newtheorem{lemma}{Lemma}
\newtheorem{corollary}{Corollary}
\newtheorem{remark}{Remark}%
\newtheorem{definition}{Definition}%

\ShortHeadings{Optimization Mechanisms in Deep Learning}
{Qi, Gong, \& Li}
\firstpageno{25}

\begin{document}

\title{Extended convexity and smoothness and their applications in deep learning}

\author{\name Binchuan Qi \email 2080068@tongji.edu.cn \\
        \AND
       \name Wei Gong (Corresponding author)  \email weigong@tongji.edu.cn\\ 
       \AND 
       \name Li Li \email lili@tongji.edu.cn\\
       \addr College of Electronics and Information Engineering, Tongji University, Shanghai, 201804, China}

\maketitle

\begin{abstract}
Classical assumptions like strong convexity and Lipschitz smoothness often fail to capture the nature of deep learning optimization problems, which are typically non-convex and non-smooth, making traditional analyses less applicable.
This study aims to elucidate the mechanisms of non-convex optimization in deep learning by extending the conventional notions of strong convexity and Lipschitz smoothness. By leveraging these concepts, we prove that, under the established constraints, the empirical risk minimization problem is equivalent to optimizing the local gradient norm and structural error, which together constitute the upper and lower bounds of the empirical risk.
Furthermore, our analysis demonstrates that the stochastic gradient descent (SGD) algorithm can effectively minimize the local gradient norm. Additionally, techniques like skip connections, over-parameterization, and random parameter initialization are shown to help control the structural error.
Ultimately, we validate the core conclusions of this paper through extensive experiments. Theoretical analysis and experimental results indicate that our findings provide new insights into the mechanisms of non-convex optimization in deep learning. 
\end{abstract}

\section{Introduction}\label{sec:introduction}

In deep learning, first-order optimization methods, including gradient descent (GD)~\cite{Carmon2017LowerBF}, stochastic gradient descent (SGD)~\cite{Ghadimi2013StochasticFA}, and advanced stochastic variance reduction techniques~\cite{Fang2018SPIDERNN,Wang2019SpiderBoostAM}, are widely employed to solve the non-convex and non-Lipschitz smooth optimization problems that arise during model training~\cite{Nemirovski2008RobustSA,Ge2015EscapingFS}. Within the classical optimization framework, these algorithms are typically analyzed for their ability to converge to stationary points, which may correspond to local minima, maxima, or saddle points~\cite{Lee2016GradientDO,Jentzen2018StrongEA}. However, empirical evidence consistently shows that, in practice, first-order methods exhibit remarkable efficiency in locating approximate global optima for the highly complex and non-convex loss landscapes encountered in deep learning~\cite{He2015DeepRL,Zhang2016UnderstandingDL,Bottou2016OptimizationMF,Belkin2018ReconcilingMM,Zhang2021UnderstandingDL}. This observed success challenges the limitations of classical optimization theory and calls for new theoretical insights into the mechanisms driving effective deep learning optimization.

To better understand the dynamics of training deep neural networks (DNNs), recent studies have explored the infinite-width regime through the lens of the Neural Tangent Kernel (NTK)~\cite{Jacot2018NeuralTK,Du2018GradientDP,AllenZhu2018ACT,Arora2019FineGrainedAO,mohamadi2023fast,liu2024ntk} and mean-field analysis~\cite{Sirignano2018MeanFA,Mei2018AMF,Chizat2018OnTG,nguyen2023rigorous,kim2024transformers}. These approaches provide valuable theoretical guarantees on convergence and training behavior. However, real-world DNNs operate at finite widths, limiting the practical applicability of infinite-width approximations~\cite{Seleznova2020AnalyzingFN,SeleznovaK21,Vyas2023EmpiricalLO}. Moreover, while it is widely believed that stochastic optimization introduces implicit regularization effects~\cite{Dauphin2014IdentifyingAA,Keskar2016OnLT,Chaudhari2017StochasticGD,Lee2019FirstorderMA,ding2024flat,le2024gradient,zou2024flatten} and that over-parameterization facilitates training~\cite{Du2018GradientDP,Yun2018SmallNI,Chizat2018OnLT,Arjevani2022AnnihilationOS}, a comprehensive theoretical explanation remains elusive. A systematic review by \cite{Oneto2023DoWR} concludes that the classical frameworks of optimization and learning theory are insufficient to fully explain the effectiveness of training modern deep models.

Inspired by the duality between strong convexity and Lipschitz smoothness~\cite{kakade2009duality,zhou2018fenchel}, as formalized in Lemma~\ref{lem:convex_smooth_dual}, we extend these classical notions using the Fenchel–Young loss framework~\cite{Blondel2019LearningWF}. 
Specifically, Lemma~\ref{lem:convex_smooth_dual} reveals that both properties characterize the rate of function variation by bounding it via $\|\cdot\|_2^2$. Building upon this insight, we generalize the concepts of convexity and smoothness by replacing $\|\cdot\|_2^2$ with a broader family of norm-based power functions $\|\cdot\|^r+c$, enabling us to define $\mathcal{H}(\phi,c_\phi)$-convexity and $\mathcal{H}(\Phi,c_\Phi)$-smoothness, where $r>1$ and $c\in \mathbb{R}_{\ge 0}$. Using these generalized properties, we analyze the non-convex optimization mechanisms in deep learning under more flexible and realistic assumptions.
We show that the empirical risk minimization problem in deep learning can be reformulated as a composite optimization problem governed by two key components: the local gradient norm and the structural error. Together, these terms define the upper and lower bounds of the empirical risk. Under our framework, effective optimization corresponds to minimizing both components simultaneously.

\paragraph{Contributions}
The main contributions of this work are summarized as follows:

\begin{itemize}
    \item We generalize the classical notions of Lipschitz smoothness and strong convexity using the Fenchel–Young loss framework, introducing $\mathcal{H}(\phi,c_\phi)$-convexity and $\mathcal{H}(\Phi,c_\Phi)$-smoothness, and investigate their mathematical properties.
    
    \item Under these generalized conditions, we impose practical constraints on loss functions and model architectures that align with real-world deep learning applications. We prove that under these constraints, the empirical risk minimization problem reduces to jointly optimizing the local gradient norm and structural error, which serve as tight bounds on the objective.
    
    \item We demonstrate theoretically and empirically that mini-batch SGD effectively reduces the local gradient norm. Furthermore, architectural techniques such as skip connections, random initialization, and over-parameterization contribute significantly to reducing the structural error.
    
    \item We validate our theoretical findings through comprehensive experimental studies across multiple datasets and model architectures, confirming the generality and robustness of our proposed framework.
\end{itemize}

Figure~\ref{fig:paper_outline} presents an overview of the core theorems in this paper and their logical interdependencies. Theorem~\ref{thm:st_H_eigenvalue_bound} establishes that minimizing the empirical risk $\mathcal{L}(\theta,s) - \mathcal{L}_*(s)$ is equivalent to minimizing both the local gradient norm and structural error, where $\mathcal{L}_*(s)$ denotes the global minimum of the empirical risk. Theorem~\ref{thm:general_with_class_sgd_convergence} demonstrates that SGD and its variants successfully reduce the local gradient norm. Proposition~\ref{prop:skip_con} and Theorem~\ref{thm:number_para} further establish how techniques like skip connections, dropout, and over-parameterization help control the structural error under realistic assumptions.

The core theorems of this paper and their logical relationships are illustrated in Figure~\ref{fig:paper_outline}. Theorem~\ref{thm:st_H_eigenvalue_bound} demonstrates that the empirical risk minimization problem $\min_\theta \mathcal{L}(\theta,s)-\mathcal{L}_*(s)$ is equivalent to minimizing both the local gradient norm and the structural error, where $\mathcal{L}_*(s)$ represents the global minimum value of the empirical risk. Theorem~\ref{thm:general_with_class_sgd_convergence} proves that SGD/mini-batch SGD can effectively reduce the local gradient norm. Proposition~\ref{prop:skip_con} and Theorem~\ref{thm:number_para} further establish that techniques such as skip connections, dropout, random parameter initialization, and over-parameterization can all contribute to reducing the structural error. 

\begin{figure}[!hbt]
\centering
\centerline{\includegraphics[width=0.9\columnwidth]{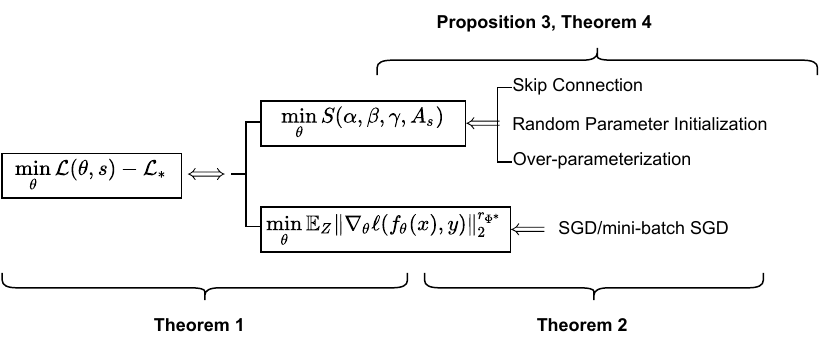}}
\caption{Logical structure of the core theorems in the paper.}
\label{fig:paper_outline}
\end{figure}
\paragraph{Organization}

The remainder of this paper is organized as follows:
 Section~\ref{sec:related_work} reviews related work.
 Section~\ref{sec:prelim} introduces foundational concepts and definitions.
 Sections~\ref{sec:def} and \ref{sec:methods} present the theoretical framework and main results.
 Section~\ref{sec:experiments} details the experimental setup and reports empirical findings.
 Section~\ref{sec:conclusion} summarizes the key contributions and discusses future directions.
 Appendix~\ref{appendix:other_related_work} provides an extended discussion of additional related work.
 Appendix~\ref{appendix:math_background} contains the mathematical background and supporting lemmas.
 Appendix~\ref{appendix:proof} includes detailed proofs of all theorems.
 Appendix~\ref{appendix:notation_table} offers a summary of notation used throughout the paper.

\section{Related Work}\label{sec:related_work}

In this section, we provide a focused review of key research lines that are closely related to the contributions of this work, namely, implicit regularization, over-parameterization, and extensions of strong convexity and Lipschitz smoothness. For complementary discussions on gradient-based optimization algorithms and Kurdyka-Łojasiewicz (KŁ) theory, we refer the interested reader to Appendix~\ref{appendix:other_related_work}.

\subsection{Implicit Regularization in SGD}

Despite its simplicity, SGD demonstrates remarkable effectiveness in solving complex and non-convex optimization problems in deep learning~\cite{Bottou2016OptimizationMF}. This has led to extensive research into the mechanisms that enable SGD to converge efficiently toward favorable solutions, even in highly non-convex landscapes.
A widely accepted hypothesis is that the inherent stochasticity of SGD introduces an implicit regularization effect, guiding the optimization trajectory toward flat minima. This behavior has been shown to influence both model generalization and fitting performance \cite{Dauphin2014IdentifyingAA,Keskar2016OnLT,Lee2019FirstorderMA,zhou2023class,ding2024flat,le2024gradient,zou2024flatten}.
One commonly adopted measure of flatness is the largest eigenvalue of the Hessian matrix of the loss function with respect to the model parameters on the training set, denoted as $\lambda_{\max}$ ~\cite{Jastrzebski2017ThreeFI,Lewkowycz2020TheLL,Dinh2017SharpMC}. However, it has been demonstrated that flatness can be arbitrarily altered through reparametrization of the model without changing its output function~\cite{dinh2017sharp}, raising questions about the validity of this metric as a reliable indicator of generalization. 
As of now, the theoretical mechanisms underlying implicit regularization have not been fully elucidated.

\subsection{Over-parameterization}

Another key factor associated with successful deep learning optimization is over-parameterization, the use of models with significantly more parameters than necessary to fit the data. Studies have shown that over-parameterization facilitates convergence to global optima and enhances generalization in DNNs \cite{Du2018GradientDP,Chizat2018OnLT,Arjevani2022AnnihilationOS}.
This insight has led to the development of the NTK theory~\cite{Jacot2018NeuralTK}, which provides a theoretical lens for analyzing the training dynamics of over-parameterized DNNs under gradient-based optimization. It is well established that infinitely wide neural networks are equivalent to Gaussian processes at initialization~\cite{Neal1996PriorsFI,Williams1996ComputingWI,Winther2000ComputingWF,Neal1995BayesianLF,Lee2017DeepNN}. In this regime, the evolution of trainable parameters during training can be characterized via the NTK, which captures how the model output responds to parameter updates.
Further studies have shown that under appropriate scaling and random initialization, the parameters of two-layer DNNs with infinite width remain close to their initial values throughout training, effectively behaving like linear models with infinite-dimensional features~\cite{Du2018GradientDP,Li2018LearningON,Du2018GradientDF,Arora2019FineGrainedAO}. This allows the training process to be analyzed using linear dynamics, enabling proofs of global convergence at a linear rate under gradient descent.
Subsequent work has extended these results to multi-layer architectures~\cite{AllenZhu2018LearningAG,Zou2018StochasticGD,AllenZhu2018ACT,mohamadi2023fast,Fu2023TheoreticalAO,liu2024ntk}, demonstrating that similar theoretical guarantees can be achieved for deeper models under idealized conditions.

An alternative approach to modeling infinitely wide DNNs is the mean-field perspective~\cite{Sirignano2018MeanFA,Mei2018AMF,Chizat2018OnTG,nguyen2023rigorous,kim2024transformers}. This line of work interprets noisy SGD as a gradient flow in the space of probability distributions over network parameters. Under suitable assumptions, this formulation establishes that noisy SGD converges to the unique global optimum of a convex objective function for two-layer networks with continuous width.

While NTK and mean-field theories provide foundational insights into the training dynamics of infinitely wide DNNs, their applicability to practical finite-width models remains limited. Recent empirical and theoretical findings~\cite{Seleznova2020AnalyzingFN,SeleznovaK21,Vyas2023EmpiricalLO} suggest that NTK-based analyses fail to accurately capture the behavior of finite-width networks trained with gradient-based methods, indicating the need for a fundamentally different conceptual framework.
Thus, despite significant progress, the interplay between over-parameterization and optimization dynamics in deep learning remains an open problem~\cite{Oneto2023DoWR}, motivating further theoretical and empirical investigations into the principles underlying effective deep learning training.

\subsection{Extensions of Strong Convexity and Lipschitz Smoothness}

Strong convexity plays a foundational role in optimization theory and its applications across various domains. It ensures unique solvability, stability, and fast convergence rates for first-order optimization methods. For an overview of its significance in pure and applied sciences, including mathematical programming, economics, game theory, and engineering, see~\cite{Alabdali2019CharacterizationsOU,Lin2003SomeEP,niculescu2006convex,noor2020higher,Mohsen2019StronglyCF,Zhu1996CoCoercivityAI,Noor2021PropertiesOH} and references therein.

Lin et al.~\cite{Lin2003SomeEP} were among the first to introduce the concept of higher-order strong convexity, motivated by applications in mathematical programs with equilibrium constraints. A function $F$ defined on a closed convex set $K$ is said to be higher-order strongly convex if there exists a constant $\mu \geq 0$ such that
\begin{equation}\label{eq:higher_order_convex}
    F(u + t(v - u)) \leq (1 - t)F(u) + tF(v) - \mu \varphi(t) \|v - u\|^p, \quad \forall u, v \in K, \; t \in [0,1], \; p \geq 1,
\end{equation}
where $\varphi(t) = t(1 - t)$. Notably, when $p = 2$, this definition reduces to the classical notion of strong convexity. Higher-order strong convexity has since found applications in diverse areas, including engineering design, economic equilibrium modeling, and multilevel games~\cite{Mohsen2019StronglyCF,Alabdali2019CharacterizationsOU,noor2020higher,Noor2021PropertiesOH}.

In classical optimization theory (e.g.,~\cite{Seidler1984ProblemCA,Nesterov2014IntroductoryLO}), the assumption of Lipschitz smoothness is central. A differentiable function $f$ is said to be Lipschitz smooth if its gradient is Lipschitz continuous, i.e.,
$$
\|\nabla f(x) - \nabla f(y)\| \le L \|x - y\| \quad \text{for all } x, y,
$$
or equivalently, $\|\nabla^2 f(x)\| \le L$ almost everywhere for some constant $L \ge 0$, known as the smoothness constant. This condition enables key theoretical results, such as the convergence guarantees of gradient descent (GD). For instance, it is well established that GD with step size $h = 1/L$ achieves an optimal rate (up to constant factors) for minimizing smooth non-convex functions~\cite{Carmon2017LowerBF}. Important examples of Lipschitz smooth functions include quadratic and logistic loss functions.

However, the Lipschitz smoothness condition is relatively restrictive, it limits functions to being upper-bounded by a quadratic growth model. As a result, many practical machine learning problems involving exponential or higher-order polynomial objectives cannot be adequately captured under this framework. To address this limitation, several generalized notions of smoothness have been proposed in the literature, allowing for broader classes of objective functions.
One notable generalization is the $(L_0, L_1)$-smoothness condition introduced by Zhang et al.~\cite{Zhang2019WhyGC}, which assumes:
$$
\|\nabla^2 f(x)\| \le L_0 + L_1 \|\nabla f(x)\|
$$
for constants $L_0, L_1 \ge 0$. This condition was inspired by empirical observations in large-scale language models and provides a more flexible growth bound that accommodates super-quadratic behaviors.
Subsequent studies have further analyzed this condition in both convex and non-convex settings~\cite{Zhang2020ImprovedAO,Qian2021UnderstandingGC,Zhao2021OnTC,crawshaw2022robustness}, and extended variance reduction techniques to accommodate $(L_0, L_1)$-smooth objectives~\cite{Reisizadeh2023VariancereducedCF}.
Chen et al.~\cite{Chen2023GeneralizedSmoothNO} proposed an alternative formulation termed $\alpha$-symmetric generalized smoothness, which offers similar expressiveness as the $(L_0, L_1)$-smoothness model.
Building upon these developments, Li et al.~\cite{Li2023ConvexAN} introduced the more general $\ell$-smoothness condition, which assumes:
$$
\|\nabla^2 f(x)\| \le \ell(\|\nabla f(x)\|)
$$
for some non-decreasing continuous function $\ell$. This formulation unifies and extends previous definitions, enabling the analysis of a wide range of non-Lipschitz-smooth objectives commonly encountered in deep learning and high-dimensional statistics.

\section{Preliminary}
\label{sec:prelim}
This section introduces the key notations, definitions, and basic setup used throughout the paper. We establish a consistent mathematical language to describe the instance space, hypothesis space, training data, empirical risk, structural matrix, structural error, and gradient norms.

\subsection{Notations}
We adopt the following notation conventions throughout this work:
\begin{itemize}

    \item The domain of a function $\Omega \colon \mathbb{R}^d \rightarrow \bar{\mathbb{R}}$ is denoted by $\mathrm{dom}(\Omega) := \{\mu \in \mathbb{R}^d \mid \Omega(\mu) < \infty \}$, where $\bar{\mathbb{R}} := \mathbb{R} \cup \{\infty\}$. The set of positive real numbers is denoted by $\mathbb{R}_{>0}$, and the set of non-negative real numbers is denoted by $\mathbb{R}_{\geq 0}$.

    \item For a proper convex function $f$, its sub-differential at $x$ is defined as:
    $$
    \partial f(x) = \left\{ s \in \mathbb{R}^n \mid f(y) - f(x) - \langle s, y - x \rangle \geq 0, \forall y \in \mathrm{dom}(f) \right\}.
    $$
    When $f$ is differentiable at $x$, we have $\partial f(x) = \{ \nabla f(x) \}$. While some activation functions (e.g., ReLU) and loss functions may exhibit points of non-differentiability, in practice these points are finite and can be handled using subgradient methods. Moreover, such functions can often be approximated by smooth alternatives, which allows us to proceed under the assumption that all relevant functions are bounded, differentiable, and closed within their open domains without loss of generality.

    \item The length of a vector $v$ is denoted $|v|$. Similarly, we denote the magnitude of a function’s output by $|f(x)|$, and the number of model parameters $\theta$ is also written as $|\theta|$.

    \item The maximum and minimum eigenvalues of a matrix $A$ are denoted by $\lambda_{\max}(A)$ and $\lambda_{\min}(A)$, respectively.

    \item The Legendre–Fenchel conjugate of a convex function $\Omega$ is given by:
    $$
    \Omega^*(\nu) := \sup_{\mu \in \mathrm{dom}(\Omega)} \left\{ \langle \nu, \mu \rangle - \Omega(\mu) \right\},
    $$
    following~\cite{Todd2003ConvexAA}. The gradient of $\Omega^*$ at $\nu$ is denoted by $\mu_\Omega^* = \nabla_\nu \Omega^*(\nu)$. In the special case where $\Omega(\mu) = \frac{1}{2} \|\mu\|_2^2$, it holds that $\mu = \mu_\Omega^*$.

\end{itemize}

\subsection{Basic Setting}
\label{subsec:basic_setting}

\paragraph{Instance Space.}  
The instance space is denoted as $\mathcal{Z} := \mathcal{X} \times \mathcal{Y}$, where $\mathcal{X}$ represents the input feature space and $\mathcal{Y}$ denotes the label set. This defines the domain of all possible inputs $x \in \mathcal{X}$ and their corresponding labels $y \in \mathcal{Y}$.

\paragraph{Hypothesis Space.}  
The hypothesis space is a collection of functions parameterized by $\Theta$, formally represented as $\mathcal{F}_{\Theta} = \{f_\theta : \theta \in \Theta\}$, where $\Theta$ denotes the parameter space. Here, $f_\theta$ (abbreviated as $f$) represents the model characterized by the parameter vector $\theta$. For a given input $x \in \mathcal{X}$, the model's output under the parameterization $\theta$ is denoted by $f_\theta(x)$ (or simply $f(x)$).

\paragraph{Training Dataset.}  
The training dataset is represented as an $n$-tuple $s^n = \{(z^{(i)})\}_{i=1}^n = \{(x^{(i)}, y^{(i)})\}_{i=1}^n$ (abbreviated as $s$), consisting of independent and identically distributed (i.i.d.) samples drawn from an unknown true distribution $\bar{q}$. The empirical probability mass function (PMF) of $Z$, derived from these samples, is denoted by $q^n$ (or simply $q$) and is defined as:
\[
q(z) := \frac{1}{n} \sum_{i=1}^n \mathbf{1}_{\{z\}}(z^{(i)}),
\]
where $\mathbf{1}_{\{z\}}(z^{(i)})$ is the indicator function that equals 1 if $z^{(i)} = z$ and 0 otherwise. By the (weak) law of large numbers, the empirical PMF $q$ converges in probability to the true distribution $\bar{q}$ as $n \to \infty$.

\paragraph{Empirical Risk.}  
The empirical risk, $\mathcal{L}(\theta, s)$, is defined as the average loss over the sample dataset:
\begin{equation}
\begin{aligned}\label{eq:emp_risk}
\mathcal{L}(\theta, s) &= \frac{1}{|s|} \sum_{(x, y) \in s} \ell(f_\theta(x), y) \\
&= \mathbb{E}_{XY \sim q} [\ell(f_\theta(X), Y)] = \mathbb{E}_{Z \sim q} [\mathcal{L}(\theta, Z)],
\end{aligned}
\end{equation}
where $\ell: \mathbb{R}^{m_f + m_y} \to \bar{\mathbb{R}}$ is the loss function that quantifies the discrepancy between the model's prediction $f_\theta(x)$ and the true label $y$. The model $f_\theta(x): \mathbb{R}^{m_x} \to \mathbb{R}^{m_f}$, parameterized by $\theta$, maps input features $x \in \mathbb{R}^{m_x}$ to predictions $f_\theta(x) \in \mathbb{R}^{m_f}$. Here, $m_x$, $m_y$, and $m_f$ denote the dimensions of the input $x$, the label $y$, and the model output $f_\theta(x)$, respectively. The loss associated with a single instance $z = (x, y)$ is denoted by $\mathcal{L}(\theta, z) = \ell(f_\theta(x), y)$.

\paragraph{Structural Matrix and Structural Error.}  
\label{def:structural_error}
We define the \textbf{structural matrix} $A_x$ corresponding to the model with input $x$ as:
\[
A_x := \nabla_\theta f_\theta(x)^\top \nabla_\theta f_\theta(x).
\]
This matrix captures the sensitivity of the model's output $f_\theta(x)$ with respect to changes in the parameter vector $\theta$.

The \textbf{structural error} of the model $f_\theta(x)$, denoted by $S(\alpha, \beta, \gamma, A_x)$, is defined as:
\begin{equation}
S(\alpha, \beta, \gamma, A_x) = \alpha D(A_x) + \beta U(A_x) + \gamma L(A_x),
\end{equation}
where:
\[
U(A_x) = -\log \lambda_{\min}(A_x), \quad L(A_x) = -\log \lambda_{\max}(A_x), \quad D(A_x) = U(A_x) - L(A_x).
\]
Here, $\alpha$, $\beta$, and $\gamma$ are weights associated with $D(A_x)$, $U(A_x)$, and $L(A_x)$, respectively.

To characterize the structural error of the dataset $s$, we introduce the following definitions:
\begin{equation}
\begin{aligned}
&\lambda_{\min}(A_s) = \min_{x \in s} \lambda_{\min}(A_x), \quad \lambda_{\max}(A_s) = \max_{x \in s} \lambda_{\max}(A_x), \\
&S(\alpha, \beta, \gamma, A_s) = \alpha D(A_s) + \beta U(A_s) + \gamma L(A_s),
\end{aligned}
\end{equation}
where:
\[
U(A_s) = -\log \lambda_{\min}(A_s), \quad L(A_s) = -\log \lambda_{\max}(A_s), \quad D(A_s) = U(A_s) - L(A_s).
\]

\paragraph{Gradient Norms.}  
Let $r > 1$ be any positive real number. We define the term:
\[
\mathbb{E}_{XY} \big[ \|\nabla_\theta \ell(f_{\theta}(X), Y)\|_2^{r} \big] \quad \text{(equivalently, } \mathbb{E}_{XY} \big[ \|\nabla_\theta \mathcal{L}(\theta, z)\|_2^{r} \big]\text{)}
\]
as the \textbf{local gradient norm}. This quantity measures the expected magnitude of the gradient of the loss function with respect to the model parameters for individual data samples $(x, y)$. In contrast, we refer to:
\[
\|\nabla_\theta \mathcal{L}(\theta, s)\|_2^{r}
\]
as the \textbf{global gradient norm}, which captures the gradient of the overall empirical risk $\mathcal{L}(\theta, s)$ with respect to the model parameters, averaged over the entire dataset $s$.

\section{Extending Strong Convexity and Smoothness via Norm Power Functions}

In this section, we introduce the concept of a norm power function and use it to generalize the classical notions of strong convexity and Lipschitz smoothness. This generalization gives rise to the concepts of $\mathcal{H}(\phi, c_\phi)$-convexity and $\mathcal{H}(\Phi, c_\Phi)$-smoothness. We further analyze their properties and implications in optimization theory.

\subsection{Norm power function}
\label{subsec:typically_homo_func}
We begin with the formal definition of a norm power function.
\begin{definition}[Norm power function]
A function $\Phi: \mathbb{R}^m \to \mathbb{R}$ is said to be a norm power function of order $r_{\Phi}$, denoted by $\Phi \in \mathcal{H}(r_{\Phi})$, if it can be expressed as:
\begin{equation}
    \Phi(\cdot) = \|\cdot\|^{r_\Phi},
\end{equation}
where $\|\cdot\|$ represents an arbitrary norm, and $r_\Phi > 1$. 

The associated normalized version for $\Phi \in \mathcal{H}(r_{\Phi})$ is given by:
\begin{equation}
    \bar{\Phi}(\mu) = (r_{\Phi} \Phi(\mu))^{1/r_{\Phi}}.
\end{equation}
\end{definition}

Common instances of norm power functions include $\|\mu\|_2^2$, $\|\mu\|_2^{r_\Phi}$, and $(\sum_{i=2}^T \|\mu\|_i)^{r_\Phi}$, where $\|\cdot\|_i$ denotes the $L_p$-norm with $p = i$. It is evident that any norm power function $\Phi(\mu)$ is convex and positively homogeneous of degree $r_\Phi$, satisfying $\Phi(a\mu) = |a|^{r_\Phi}\Phi(\mu)$ for all $a \in \mathbb{R}$.

The following lemma summarizes key properties of norm power functions and their normalized counterparts.
\begin{lemma}[Properties of norm power functions]
\label{lem:prop_thomo_fun}
Let $\Phi:\mathbb{R}^m\to \bar{\mathbb{R}} \in \mathcal{H}(r_{\Phi})$. Then:
\begin{enumerate}
    \item \label{lem:prop_thomo_fun:1} The Legendre-Fenchel conjugate $\Phi^*$ satisfies $\Phi^* \in \mathcal{H}(r_{\Phi^*})$, with $1/r_{\Phi} + 1/r_{\Phi^*} = 1$. \label{prop:n_loss_duality}
    
    \item \label{lem:prop_thomo_fun:2} The relationship between $\Phi$ and its conjugate is given by:
    \begin{equation}
        r_\Phi \Phi(\mu) = r_{\Phi^*} \Phi^*(\mu^*_\Phi) = \mu^\top \mu_\Phi^*,
    \end{equation}
    where $\mu^*_\Phi$ is the dual variable associated with $\mu$. \label{prop:homo_eq}
    
    \item \label{prop:triangle} The normalized norm power function satisfies the triangle inequality:
    \begin{equation}
        \bar{\Phi}(\mu + \nu) \leq \bar{\Phi}(\mu) + \bar{\Phi}(\nu).
    \end{equation}
    
    \item \label{lem:prop_thomo_fun:4} The product of the normalized norm power function and its conjugate satisfies:
    \begin{equation}
        \bar{\Phi}(\mu) \bar{\Phi}^*(\nu) \geq |\langle \mu, \nu \rangle|,
    \end{equation}
    where $\langle \cdot, \cdot \rangle$ denotes the inner product. 
    
    \item \label{lem:prop_thomo_fun:5} The equivalence between the $L_2$ norm and $\bar{\Phi}$ is established as:
    \begin{equation}
    \begin{aligned}
    m^{-1/2} \left(\sum_{i=1}^m \bar{\Phi}^*(\mathbf{e}_i)^2\right)^{-1/2} \|\mu\|_2 
    &\leq \bar{\Phi}(\mu) \\
    &\leq \left(\sum_{i=1}^m \bar{\Phi}(\mathbf{e}_i)^2\right)^{1/2} \|\mu\|_2,
    \end{aligned}
    \end{equation}
    where $\mathbf{e}_i$ denotes a standard basis vector (one-hot vector).
\end{enumerate}
\end{lemma}

The proof of this lemma can be found in Appendix~\ref{appendix:proof_prop_thomo_fun}.

\subsection{\texorpdfstring{$\mathcal{H}(\phi,c)$}{}-Convexity and \texorpdfstring{$\mathcal{H}(\Phi,s)$}{}-Smoothness}
\label{subsec:extend_cs}

\subsubsection{Definitions and Properties}
\label{subsubsec:def_s_c}
Building upon the concept of norm power functions, we introduce the notions of $\mathcal{H}(\phi,c)$-convexity and $\mathcal{H}(\Phi,s)$-smoothness. These generalized definitions provide a flexible framework for analyzing the structural properties of functions in optimization and machine learning.

\begin{definition}
\label{def:g_convex_smooth}
We define the following extended notions of convexity and smoothness:
\begin{itemize}
    \item \textbf{$\mathcal{H}(\phi,c_\phi)$-convex function}. A function $G: \mathbb{R}^m \to \mathbb{R}$ is said to be an $\mathcal{H}(\phi,c_\phi)$-convex function if it satisfies:
    \begin{equation}
    \phi(\mu-\nu) - c_\phi \leq d_G(\mu, \nu_G^*), \quad \forall \mu \in \mathbb{R}^m, \, \forall \nu \in \mathbb{R}^m,
    \end{equation}
    where $c_\phi \in \mathbb{R}_{\geq 0}$ is the \textbf{convex relaxation factor}, and $\phi \in \mathcal{H}(r_\phi)$ is a norm power function. Here, $r_\phi$ is referred to as the \textbf{convex order}. If $c_\phi = 0$, the function $G$ is called a \textbf{$\mathcal{H}(\phi)$-convex function}, i.e.,
    \begin{equation}
    \phi(\mu-\nu) \leq d_G(\mu, \nu_G^*), \quad \forall \mu \in \mathbb{R}^m, \, \forall \nu \in \mathbb{R}^m.
    \end{equation}

    \item \textbf{$\mathcal{H}(\Phi,c_\Phi)$-smooth function}. A function $G: \mathbb{R}^m \to \mathbb{R}$ is said to be an $\mathcal{H}(\Phi,c_\Phi)$-smooth function if it satisfies:
    \begin{equation}
    d_G(\mu, \nu_G^*) \leq \Phi(\mu-\nu) + c_\Phi, \quad \forall \mu \in \mathbb{R}^m, \, \forall \nu \in \mathbb{R}^m,
    \end{equation}
    where $c_\Phi \in \mathbb{R}_{\geq 0}$ is the \textbf{smooth relaxation factor}, and $\Phi \in \mathcal{H}(r_\Phi)$ is a norm power function. Here, $r_\Phi$ is referred to as the \textbf{smooth order}. If $c_\Phi = 0$, the function $G$ is called a \textbf{$\mathcal{H}(\Phi)$-smooth function}, i.e.,
    \begin{equation}
    d_G(\mu, \nu_G^*) \leq \Phi(\mu-\nu), \quad \forall \mu \in \mathbb{R}^m, \, \forall \nu \in \mathbb{R}^m.
    \end{equation}
\end{itemize}
\end{definition}

\begin{remark}
While the Bregman divergence is widely used in machine learning and optimization, and it holds that $B_F(\mu, \nu) = d_F(\mu, \nu_\Omega^*)$ when $F$ is convex, the Bregman divergence inherently requires $F$ to be convex. In contrast, the Fenchel-Young loss does not impose this restriction, allowing it to encompass a broader range of scenarios. By adopting the Fenchel-Young loss as a framework to constrain the loss function and model, this paper achieves greater generality and flexibility in its analysis.

The parameters $c_\phi$ and $c_\Phi$ introduce relaxation factors, which allow for deviations from strict convexity or smoothness, making the framework applicable to a wider variety of practical settings.
\end{remark}
We now explore the relationship between classical notions of strong convexity and Lipschitz smoothness and their generalized counterparts, $\mathcal{H}(\phi,c_\phi)$-convexity and $\mathcal{H}(\Phi,c_\Phi)$-smoothness. This connection reveals that the classical definitions are special cases within our broader framework.
A function $G$ that is Lipschitz smooth corresponds to an $\mathcal{H}(\Phi)$-smooth function with $\Phi(\cdot) = \|\cdot\|_2^2$. 
Strongly convex functions and higher-order strongly convex functions~\cite{Lin2003SomeEP} are specific instances of $\mathcal{H}(\phi)$-convex functions. In particular: For strongly convex functions, $\phi(\cdot) = \|\cdot\|_2^2$; For higher-order strongly convex functions, $\phi(\cdot) = \|\cdot\|^p$, where $p > 2$.
For a twice-differentiable function $G$ with a positive definite Hessian matrix, the Taylor mean value theorem implies:
\[
d_G(\mu, \nu_G^*) = \frac{1}{2} (\mu - \nu)^\top \nabla_\xi^2 G(\xi) (\mu - \nu),
\]
where $\xi = t\mu + (1-t)\nu$ for $t \in [0, 1]$. Using Lemma~\ref{lem:eigen_bound}, it follows that:
\[
\frac{\lambda_{\min}}{2} \|\mu - \nu\|_2^2 \leq d_G(\mu, \nu_G^*) \leq \frac{\lambda_{\max}}{2} \|\mu - \nu\|_2^2,
\]
where $\lambda_{\min}$ and $\lambda_{\max}$ are the smallest and largest eigenvalues of the Hessian matrix $\nabla_\xi^2 G(\xi)$, respectively. This result shows that $G$ is both:  $\mathcal{H}(\lambda_{\max}\|\cdot\|_2^2/2)$-smooth, as the upper bound on $d_G(\mu, \nu_G^*)$ depends on $\lambda_{\max}$, and  $\mathcal{H}(\lambda_{\min}\|\cdot\|_2^2/2)$-convex, as the lower bound on $d_G(\mu, \nu_G^*)$ depends on $\lambda_{\min}$.

It is well established that a wide range of commonly used loss functions in statistics and machine learning, including CrossEntropy loss, Mean Squared Error (MSE) loss, and Perceptron loss, can be categorized under the framework of Fenchel-Young losses $d_F(y, f_\theta(x))$ through appropriate design of the function $F$~\cite{Blondel2019LearningWF}. To further investigate the generalized convexity and smoothness of Fenchel-Young losses, we present the following lemma:
\begin{lemma}
    \label{lem:fenchel_young_condition}
Given a function $g: \mathbb{R}^m \to \mathbb{R}$, let $G(\mu) = d_g(\mu, s)$, where $s$ is a constant vector. Then, it holds that:
\begin{equation}
    d_G(\mu,\nu_G^*)=d_g(\mu,\nu^*_g).
\end{equation}
\end{lemma}
The proof of this lemma is provided in Appendix~\ref{appendix:proof_fenchel_young_condition}. 
The above lemma demonstrates that the generalized smoothness and convexity of the Fenchel-Young loss $d_g(\mu, s)$ can be analyzed through its generating function $g$. Leveraging Lemma~\ref{lem:fenchel_young_condition} and the examples provided in Table~\ref{tab:fy_losses_examples}, we observe the following:

\begin{itemize}
    \item \textbf{MSE loss}: The MSE loss is both $\mathcal{H}(\|\cdot\|_2^2)$-convex and $\mathcal{H}(\|\cdot\|_2^2)$-smooth. This aligns with the classical understanding of the MSE loss as being strongly convex and Lipschitz-smooth.
    \item \textbf{CrossEntropy loss}: Based on Lemma~\ref{lem:pinsker} and Lemma~\ref{lem:kl_upper_bound}, the CrossEntropy loss and its equivalent Kullback-Leibler (KL) divergence are both $\mathcal{H}(\phi)$-convex and $\mathcal{H}(\Phi)$-smooth. Specifically:
    \begin{itemize}
        \item The smoothness parameter $\Phi(\cdot)$ is given by:
        \[
        \Phi(\cdot) = \frac{1}{\inf_{x \in \mathcal{X}} \mathrm{Softmax}(f_\theta(x))} \|\cdot\|_2^2,
        \]
        where $\mathrm{Softmax}(\cdot)$ denotes the Softmax function. 
        \item The convexity parameter $\phi(\cdot)$ is expressed as:
        \[
        \phi(\cdot) = \frac{1}{2\ln 2} \|\cdot\|_1^2.
        \]
    \end{itemize}
\end{itemize}

\subsubsection{Connecting Gradients and Optimal Solutions}
The following theorem establishes upper and lower bounds for the deviation of $G(\mu)$ from its global minimum $G_*$, thereby linking the gradient behavior of the function to its optimal solution.

\begin{lemma}
\label{lem:h_bound}
Let $G_* = \min_{\mu} G(\mu)$, $\mathcal{G}_* = \{\mu \mid G(\mu) = G_*\}$, and $\mu_* \in \mathcal{G}_*$ denote the set of minimizers and a specific minimizer, respectively. The following results hold:

\textbf{1. }\label{lem:h_bound:1} If $G(\mu)$ is $\mathcal{H}(\Phi,c_\Phi)$-smooth, then:
\[
\Phi^*(\mu_G^*) - c_\Phi \leq G(\mu) - G_* \leq \Phi(\mu - \mu_*) + c_\Phi.
\]

\textbf{2. }\label{lem:h_bound:2} If $G(\mu)$ is $\mathcal{H}(\phi,c_\phi)$-convex, then:
\[
\phi(\mu - \mu_*) - c_\phi \leq G(\mu) - G_* \leq \phi^*(\mu_G^*) + c_\phi.
\]

\textbf{3. }\label{lem:h_bound:3} If $G(\mu)$ is both $\mathcal{H}(\Phi,c_\Phi)$-smooth and $\mathcal{H}(\phi,c_\phi)$-convex, then:
\[
\Phi^*(\mu_G^*) - c_\Phi \leq G(\mu) - G_* \leq \phi^*(\mu_G^*) + c_\phi.
\]
\end{lemma}
The proof of this lemma is provided in Appendix~\ref{appendix:proof_h_bound}.
\begin{remark}
It is worth noting that under the specific choice of the desingularizing function $\psi(s) := \frac{1}{1-\theta}s^{1-\theta}$ and the condition $\nabla f(\bar{\mu}) = 0$, the inequality describing the Kurdyka-Łojasiewicz (KŁ) property in~\eqref{eq:kl_property} takes the form:
\[
|\nabla f(u)| > (f(u) - f(\bar{\mu}))^\theta.
\]
The upper bound derived in Conclusion 2 of Lemma~\ref{lem:h_bound} may correspond to a special case of the KŁ property. However, this connection is not the focus of this paper and will not be further analyzed here.
\end{remark}

Building upon the above lemma, we derive the following corollary:
\begin{corollary}
If $G(\mu)$ is both $\mathcal{H}(\phi)$-convex and $\mathcal{H}(\Phi)$-smooth, where $\Phi(\mu) = \frac{k_\Phi}{2}\|\mu\|_2^2$ and $\phi(\mu) = \frac{k_\phi}{2}\|\mu\|_2^2$, with $k_\Phi, k_\phi \in \mathbb{R}_{>0}$, then it follows that:
\begin{equation}
    \frac{\|\mu_G^*\|_2^2}{2 k_\Phi} \leq G(\mu) - G_* \leq \frac{\|\mu_G^*\|_2^2}{2 k_\phi}.
\end{equation}
\end{corollary}
This corollary characterizes $G(\mu)$ as a Lipschitz-smooth ($L = k_\Phi$) and $c$-strongly convex ($c = k_\phi$) function, providing a clear connection between the generalized framework and classical optimization theory. These results highlight the utility of the $\mathcal{H}(\phi,c_\phi)$-convexity and $\mathcal{H}(\Phi,c_\Phi)$-smoothness framework in analyzing optimization landscapes and provide a foundation for understanding the interplay between gradients and optimal solutions in practical machine learning scenarios.

\section{Reformulation of Empirical Risk Minimization}
\label{sec:def}
In this section, we reformulate the empirical risk minimization (ERM) problem by incorporating the concepts of $\mathcal{H}(\phi,c_\phi)$-convexity, $\mathcal{H}(\Phi,s_\Phi)$-smoothness, and $\mathcal{H}(\Omega,c_\Omega)$-smoothness to impose constraints on both the loss function and the model. 
We define the empirical risk minimization problem and the constraints considered in this paper as follows:
\begin{definition}[Empirical Risk Minimization]
    \label{def:d_s_com_opt}
    We define the optimization problem studied in this paper as:
    \begin{equation}\label{eq:obj_st}
        \begin{aligned}
            \min_\theta \mathcal{L}(\theta, s).
        \end{aligned}
    \end{equation}
    The conditions that the loss function and the model must satisfy are as follows:
    \begin{itemize}
        \item \textbf{Condition 1:}~\label{con:1} $\ell(f_\theta(x), y)$ ($L(\theta,z)$) is $\mathcal{H}(\phi,c_\phi)$-convex and $\mathcal{H}(\Phi,c_\Phi)$-smooth with respect to the model's output $f_\theta(x)$.
        \item \textbf{Condition 2:}~\label{con:2} $\ell(f_\theta(x), y)$ is $\mathcal{H}(\Omega,c_\Omega)$-smooth with respect to $\theta$.
    \end{itemize}
\end{definition}

Next, we illustrate that the optimization problems corresponding to training deep models in practical scenarios fall under the category of the optimization problem defined above. 
To illustrate this, we consider Condition 2 as an example. Specifically, we can construct $\Omega(\cdot) = a\|\cdot\|_p^{r_\Omega}+c_\Omega$, where $a \geq 0$, based on the $L_p$-norm. Clearly, for most cases, by appropriately increasing the parameters $a$, $r_\Omega$ and $c_\Omega$, we can ensure that:
\begin{equation}
    \begin{aligned}
        d_{G_z}(\theta_1, (\theta_2)_{G_z}^*) \leq a\|\theta_2 - \theta_1\|_p^{r_\Omega}+c_\Omega,
    \end{aligned}
\end{equation}
where $G_z(\theta)=\ell(f_\theta(x),y)$.
This flexibility allows $\Omega(\cdot)$ to effectively bound the Fenchel-Young loss $d_{G_z}(\theta_1, (\theta_2)_{G_z}^*)$ for a wide range of functions $G_z(\theta)$. Similarly, we can also construct $\Phi(\cdot) = c\|\cdot\|_p^{r_\Phi} + c_\Phi$, $\phi(\cdot)=d\|\cdot\|_p^{r_\phi}-c_\phi$ to satisfy the conditions. As discussed in Section~\ref{subsubsec:def_s_c}, the MSE loss, CrossEntropy loss, and its equivalent KL divergence are all $\mathcal{H}(a\|\cdot\|_2^2)$-smooth, where $a \in \mathbb{R}_{> 0}$. Additionally, the Fenchel-Young loss generated by a convex function $g$ with bounded eigenvalues of its Hessian matrix is also $\mathcal{H}(a\|\cdot\|_2^2)$-smooth. Moreover, some uncommon strongly convex loss functions, such as $\|\cdot\|_2^8$, are both $\mathcal{H}(\phi, c_\phi)$-convex and $\mathcal{H}(\Phi, c_\Phi)$-smooth. However, since these functions are rarely used in practical deep learning applications, they will not be discussed further in this paper. In summary, the loss functions constrained by Condition 1 and Condition 2 encompass a wide range of commonly used loss functions in practice. This ensures that the conclusions derived based on them have broad applicability across various scenarios in statistics and machine learning.

\section{Optimization Mechanism of Deep Learning}
\label{sec:methods}

In this section, we analyze the optimization mechanisms underlying deep learning models. We first establish that the empirical risk minimization problem~\ref{def:d_s_com_opt} is equivalent to jointly optimizing two key components: the local gradient norm and the structural error. Subsequently, we demonstrate how SGD effectively minimizes the local gradient norm, while techniques such as skip connections, over-parameterization, and random parameter initialization play a crucial role in controlling the structural error.

\subsection{Bounds on Empirical Risk}

The following theorem establishes bounds that characterize the global optimum of the empirical risk minimization problem defined in Definition~\ref{def:d_s_com_opt}.
\begin{theorem}
   \label{thm:st_H_eigenvalue_bound}
Let $\mathcal{L}_*(z) = \min_{\theta} \mathcal{L}(\theta, z)$, $\mathcal{L}_*(s) = \min_{\theta} \mathcal{L}(\theta, s)$. Define the constants:
\[
C_\Phi = \frac{1}{r_{\Phi^*}} {m_f}^{-r_{\Phi^*}/2} \Big(\sum_{i=1}^{m_f} \bar{\Phi}(\mathbf{e}_i)^2\Big)^{-r_{\Phi^*}/2}, \quad 
C_\phi = \frac{1}{r_{\Phi^*}} \Big(\sum_{i=1}^{m_f} \bar{\phi}^*(\mathbf{e}_i)^2\Big)^{r_{\Phi^*}/2}.
\]
Individual Sample Bound:
If $\lambda_{\min}(A_x) \neq 0$, we have:
\begin{equation}
\begin{aligned}
C_\Phi \frac{\|\nabla_\theta \ell(f_\theta(x), y)\|_2^{r_{\Phi^*}}}{\lambda_{\max}(A_x)^{r_{\Phi^*}/2}} 
-c_\Phi\leq \mathcal{L}(\theta, z) - \mathcal{L}_*(z) 
\leq C_\phi \frac{\|\nabla_\theta \ell(f_\theta(x), y)\|_2^{r_{\Phi^*}}}{\lambda_{\min}(A_x)^{r_{\Phi^*}/2}}+c_\phi.
\end{aligned}
\end{equation}
Dataset-Wide Bound:
If $\lambda_{\min}(A_s) \neq 0$, we have:
\begin{equation}
\begin{aligned}
\frac{C_\Phi \mathbb{E}_Z \big[\|\nabla_\theta \ell(f_\theta(X), Y)\|_2^{r_{\Phi^*}}\big]}{\lambda_{\max}(A_s)^{r_{\Phi^*}/2}} -c_\Phi
\leq \mathcal{L}(\theta, s) - \mathcal{L}_*(s) 
\leq \frac{C_\phi \mathbb{E}_Z \big[\|\nabla_\theta \ell(f_\theta(X), Y)\|_2^{r_{\Phi^*}}\big]}{\lambda_{\min}(A_s)^{r_{\Phi^*}/2}}+c_\phi.
\end{aligned}
\end{equation}
\end{theorem}
The proof can be found in Appendix~\ref{appendix:proof_st_H_eigenvalue_bound}.
Given that $C_\phi$ and $C_\Phi$ are constants when $\phi$ and $\Phi$ are specified, the above theorem indicates that by optimizing the local gradient norm $\mathbb{E}_Z \left[ \|\nabla_\theta \ell(f_\theta(X), Y)\|^{r_{\Phi^*}}_2 \right]$ and the structural error $S(\alpha, \beta, \gamma, A_s)$, we can effectively reduce $\mathcal{L}(\theta, s) - \mathcal{L}_*(s)$. The experiments in this paper also confirm that the optimization mechanism in deep learning, as described in the aforementioned theorem, is achieved by reducing the local gradient norm and the structural error to control the upper and lower bounds of the empirical risk. 

\subsection{Local Gradient Norm Minimization by SGD}
\label{subsec:convergence}

In this subsection, we aim to prove that SGD, including its batch update variant, can effectively minimize the local gradient norm:
\[
\mathbb{E}_Z \left[ \|\nabla_\theta \ell(f_{\theta}(x), y)\|^{r_{\Phi^*}}_2 \right].
\]
The specific approach is as follows: Since $\ell(f_{\theta}(x), y)$ is $\mathcal{H}(\Omega,c_\Omega)$-smooth with respect to $\theta$, we first demonstrate that the SGD algorithm can ensure the convergence of $\mathbb{E}_Z \left[ \|\nabla_\theta \ell(f_{\theta}(x), y)\|^{r_{\Omega^*}}_2 \right]$. 
Because both $\mathbb{E}_Z \left[ \|\nabla_\theta \ell(f_{\theta}(x), y)\|^{r_{\Phi^*}}_2 \right]$ and $\mathbb{E}_Z \left[ \|\nabla_\theta \ell(f_{\theta}(x), y)\|^{r_{\Omega^*}}_2 \right]$ measure the distance of the gradient $\nabla_\theta \ell(f_{\theta}(x), y)$ from the origin, their optimization effects are equivalent. Therefore, the SGD algorithm will also reduce the local gradient norm $\mathbb{E}_Z \left[ \|\nabla_\theta \ell(f_{\theta}(x), y)\|^{r_{\Phi^*}}_2 \right]$.

The SGD method with a constant stepsize $\alpha$ is formalized by the following update rule~\label{def:basic-sgd}:
\begin{equation}
    \theta_{k+1} = \theta_k - \alpha \nabla_\theta \mathcal{L}(\theta_k, s_k).
\end{equation}
Here, $\mathcal{L}(\theta_k, s_k) = \frac{1}{|s_k|} \sum_{z \in s_k} \ell(f_{\theta_k}(x), y)$; $\theta_k$ denotes the parameter value at iteration $k$; $s_k$ is a subset (mini-batch) sampled from the dataset $s$; and $|s_k|$ corresponds to the batch size.
We use $s \setminus s_k$ to represent the set of all elements in $s$ that are not in $s_k$. Consequently, we have $|s \setminus s_k| = |s| - |s_k|$. To describe the impact of a single batch update on samples outside the batch, we define the gradient correlation factor.
\begin{definition}[Gradient Correlation Factor]
\label{def:gradient_coor_factor}
After each update based on $s_k$, the objective function $\mathcal{L}(\theta,s)$ satisfies the following inequality:
$$
L(\theta_{k+1}, s \setminus s_k) - L(\theta_k, s \setminus s_k) \leq M,
$$
where $M \in \mathbb{R}_{\geq 0}$ quantifies the impact of a single batch update on the samples outside the batch. We refer to $M$ as the gradient correlation factor. 
\end{definition}
Generally, the smaller the batch size $|s_k|$, the larger the total number of samples $|s|$, and the greater the number of model parameters, especially non-shared parameters (such as those parameters in embedding modules), the smaller the value of $M$. 

The following theorem delineates the convergence behavior of SGD within the $\mathcal{H}(\Omega,c_\Omega)$-smooth setting. It is important to note that, in the framework of convex optimization, the convergence analysis of SGD typically focuses on $\|\nabla_\theta \mathcal{L}(\theta,s)\|_2^2$, measuring the norm of the gradient of the loss function with respect to the parameters. However, in this context, our analysis targets a different aspect: the convergence properties of $\mathbb{E} \|\nabla_\theta \mathcal{L}(\theta, s_k)\|_2^{r_{\Phi^*}}$, which involves examining the expected value of the gradient norm raised to the power of $r_{\Phi^*}$ for mini-batch updates.
\begin{theorem}\label{thm:general_with_class_sgd_convergence}
Let $n=|s|$, $m=|s_k|$, and $\|x\|_\Omega=r_{\Omega^*}^{-1/r_{\Omega^*}}\frac{(\|x\|_2^2)}{\bar{\Omega}(x)}$.  
Given that $\mathcal{L}(\theta,z) = \ell(f_{\theta}(x),y)$ is $\mathcal{H}(\Omega,c_\Omega)$-smooth with respect to $\theta$, the application of SGD as defined in~\ref{def:basic-sgd} leads to the following conclusions:
\begin{enumerate}
    \item The optimal learning rate is achieved when 
    $$
    \alpha = \left(\frac{\|\nabla_\theta \mathcal{L}(\theta_k,s_k)\|_2^2}{r_\Omega\Omega(\nabla_\theta \mathcal{L}(\theta_k,s_k))}\right)^{r_{\Omega^*}-1},
    $$
    and it follows that
    \begin{equation}
       \mathcal{L}(\theta_k, s_k) - \mathcal{L}(\theta_{k+1}, s_k)  \geq \|\nabla_\theta \mathcal{L}(\theta_k, s_k)\|^{r_{\Omega_*}}_\Omega-c_\Omega.
    \end{equation}

    \item The number of steps required for SGD to achieve 
    $$
    \mathbb{E} \|\nabla_\theta \mathcal{L}(\theta, s_k)\|_2^{r_{\Omega^*}} \leq \varepsilon^{r_{\Omega^*}} + \frac{n-m}{m}M+c_\Omega
    $$
    is $\mathcal{O}(\varepsilon^{-r_{\Omega^*}})$.

    \item There exists a constant $\gamma \in \mathbb{R}_{>0}$ such that 
$$
\frac{1}{\gamma}\|\mu\|_2^{r_{\Phi^*}} \leq \|\mu\|_\Omega^{r_{\Phi^*}}, \quad \forall x \in \mathbb{R}^{m_f}.
$$

When $m = 1$ (i.e., a single sample per batch), the condition
$$
\frac{1}{\gamma}\mathbb{E}_Z \|\nabla_\theta \ell(f_{\theta}(x),y)\|_2^{r_{\Omega^*}} \leq \mathbb{E}_Z \|\nabla_\theta \ell(f_{\theta}(x),y)\|_\Omega^{r_{\Omega^*}} \leq \varepsilon^{r_{\Omega^*}} + (n-1)M+c_\Omega
$$
is achieved in $\mathcal{O}(1/\varepsilon^{r_{\Omega^*}})$ steps. 

\end{enumerate}
\end{theorem}
The proof of this theorem is detailed in Appendix \ref{appendix:proof_general_with_class_sgd_convergence}.
The aforementioned theorem demonstrates that when the batch size is equal to 1, the SGD algorithm can effectively control $ \mathbb{E}_Z \|\nabla_\theta \ell(f_{\theta}(x),y)\|_2^{r_{\Omega^*}} $, which is equivalent to controlling the local gradient norm $\mathbb{E}_Z \left[ \|\nabla_\theta \ell(f_{\theta}(x), y)\|^{r_{\Omega^*}}_2 \right]$. This result highlights that the use of SGD is effective for minimizing the local gradient norm. Furthermore, reducing the batch size and the gradient correlation factor $M$ can facilitate the convergence of the local gradient norm.

\subsection{Structural Error Minimization}

SGD algorithms ensure the minimization of the local gradient norm. In this subsection, we analyze another critical factor in the non-convex optimization mechanism of deep learning: structural error. We demonstrate that the model's architecture, the number of parameters, and the independence among parameter gradients can all be leveraged to reduce structural error.

\subsubsection{Skip Connection}

During the training of neural networks, the magnitudes of the elements in the gradient vector $\nabla_\theta f_{\theta}(x)$ diminish alongside the reduction in backpropagated errors. According to Theorem~\ref{thm:st_H_eigenvalue_bound}, under conditions where the local gradient norm is held constant, the loss is controlled by the eigenvalues of the structural matrix. Specifically, larger eigenvalues correlate with smaller loss.
Consequently, one direct approach to mitigating structural errors lies in augmenting the eigenvalues of the structural matrix through modifications in network architecture.
From this perspective, Residual blocks~\cite{He2015DeepRL} represent a classic and highly effective architecture. By introducing skip connections, they are able to effectively enhance the eigenvalues of the structural matrix.

Here, we present a concise mathematical explanation. Following the random initialization of model parameters or upon reaching a stable convergence state, the gradients of the model $f_{\theta}(x)$ with respect to its parameters often tend toward zero. Suppose that $f_\theta$ is composed of $k$ sequential blocks, where the output of each block is represented by a function $h^i$, and the final output satisfies $h^{(k)} = f_{\theta}(x)$. For these $k$ blocks, we introduce skip connections, thereby constructing a new model $g_\theta(x)$.  
Using the chain rule, we have:
\begin{equation}
\begin{aligned}
    \nabla_{\theta^{(j)}} f_{\theta}(x) &= \nabla_{h^{(k-1)}} f_{\theta}(x) \nabla_{h^{(k-2)}} h^{(k-1)} \cdots \nabla_{\theta^{(j)}} h^{(j)} \\
    &= \left( \prod_{i=j}^{k-1} \nabla_{h^{(i)}} h^{(i+1)} \right) \nabla_{\theta^{(j)}} h^{(j)}, \\
    \nabla_{\theta^{(j)}} g_\theta(x) &= \left( \nabla_{h^{(k-1)}} f_{\theta}(x) + I \right) \left( \nabla_{h^{(k-2)}} h^{(k-1)} + I \right) \cdots \nabla_{\theta^j} h^{(j)} \\
    &= \left( \prod_{i=j}^{k-1} \left( \nabla_{h^{(i)}} h^{(i+1)} + I \right) \right) \nabla_{\theta^{(j)}} h^{(j)},
\end{aligned}
\end{equation}
where $\theta^{(j)}$ represents the parameters corresponding to the $j$-th block, and $I$ is the identity matrix. It is evident that when the elements of $\nabla_{h^{(i)}} h^{(i+1)}$ are close to zero, the elements of $\nabla_{\theta^j} f_{\theta}(x)$ also approach zero, causing the eigenvalues of the structural matrix to become close to zero as well. However, with the introduction of skip connections, $\nabla_{\theta^{(j)}} g_\theta(x) \approx \nabla_{\theta^{(j)}} h^{(j)}$, which prevents the decay caused by the multiplication of gradients in the chain rule. skip connections prevent the eigenvalues of the structural matrix from becoming excessively small due to SGD, which would otherwise lead to significant structural errors. Therefore, we can derive the following proposition:

\begin{proposition}
    \label{prop:skip_con}
    Skip connections are beneficial for reducing structural error.
\end{proposition}

\subsubsection{Parameter Number and Independence}
The number of parameters in a model, along with the independence of these parameters, can also be leveraged to reduce structural error. Our analysis is grounded in the following condition:
\begin{definition}[Gradient independence condition (GIC)]
\label{def:grad_indep_con}
Each column of $\nabla_\theta f(x)$ is approximated as uniformly sampled from a ball $\mathcal{B}^{|\theta|}_\epsilon: \{x \in \mathbb{R}^{|\theta|}, \|x\|_2 \le \epsilon\}$.
\end{definition}
This condition essentially posits that the derivatives of the model's output with respect to its parameters are finite and approximately independent. 
In practical scenarios, gradients computed during the training of neural networks are indeed finite due to mechanisms such as gradient clipping and the inherent properties of activation functions used (e.g., ReLU, sigmoid). Furthermore, under certain conditions, such as when using techniques like dropout, it is reasonable to assume that the dependence among parameter gradients is weakened. 
In the experimental section~\ref{sec:experiments}, we observe that under conditions of large-scale parameter counts, the GIC~\ref{def:grad_indep_con} is generally satisfied immediately after random initialization of the parameters. 
However, as the training process begins, backpropagated errors are relatively large, leading to significant changes in the parameters. During this phase, our experimental results suggest that the this condition  may not hold. 
As training continues and the model enters a more stable phase where the magnitude of backpropagated errors decreases, experimental results imply that the condition becomes approximately valid again.
Therefore, this condition serves as a simplification of the model's characteristics, which is reasonable and largely consistent with practical scenarios under certain circumstances. 

Below, we will derive a theorem based on this condition to describe how the number of parameters and inter-parameter correlations affect the structural error. 
\begin{theorem}
\label{thm:number_para}
If the model satisfies the GIC~\ref{def:grad_indep_con} and $|f_\theta(x)|\le |\theta|-1$, then with probability at least $1-O(1 / |f_\theta(x)|)$ the following holds:
\begin{equation}
    \begin{aligned}
    L(A_x)\leq U(A_x)&\le -\log (1-\frac{2 \log |f_\theta(x)|}{|\theta|})^2-\log \epsilon^2,\\
      D(A_x)=U(A_x&)-L(A_x)\le \log (Z(|\theta|,|f_\theta(x)|)+1),
    \end{aligned}
\end{equation}
where $Z(|\theta|,|f_\theta(x)|)=\frac{2|y|\sqrt{6 \log |f_\theta(x)|}}{\sqrt{|\theta|-1}} (1-\frac{2 \log |f_\theta(x)|}{|\theta|})^2$, is a decreasing function of $|\theta|$ and an increasing function of $|f_\theta(x)|$.
\end{theorem}
The derivation for this theorem is presented in Appendix~\ref{appendix:proof_number_para}.

According to this theorem, the strategies to reduce structural error involve two key aspects: enhancing gradient independence to satisfy the GIC~\ref{def:grad_indep_con} and reducing the structural error by increasing the number of parameters $|\theta|$. 

Below, we analyze several deep learning techniques based on Theorem~\ref{thm:number_para}. 
Our examination includes insights derived from our designed experiments, which have validated several conclusions. It is important to note, however, that some of these conclusions are based on intuition and practical experience rather than rigorous theoretical derivation or comprehensive empirical validation. Despite this, we believe that presenting preliminary yet plausible analyses from a novel perspective can offer valuable insights and stimulate further discussion and research into these issues.

\textbf{Enhancing Gradient Independence}. To achieve the GIC, techniques that weaken inter-parameter correlations and promote independent gradient updates are essential. Some effective implementations include:
\begin{itemize}
    \item \textbf{Dropout}~\cite{Krizhevsky2012ImageNetCW}: By randomly deactivating neurons during training, dropout reduces co-adaptation among parameters, which in turn encourages more independent gradient updates. This process helps weaken the dependence among the columns of $\nabla_\theta f_{\theta}(x)$, thereby improving the rationality and applicability of the GIC. 
    \item \textbf{Random parameter initialization}. Proper weight initialization methods (e.g., Xavier or He initialization) ensure that gradients are well-scaled and less correlated at the start of training, facilitating adherence to the GIC.

\end{itemize}

\textbf{Increasing the Number of Parameters}. The number of parameters in a model influences not only its fitting capability but also plays a crucial role in its non-convex optimization ability.
Theorem~\ref{thm:number_para} indicates that the structural error decreases as the number of parameters increases, provided that the GIC is satisfied. This conclusion aligns with the pivotal viewpoint that over-parameterization plays a crucial role in the exceptional non-convex optimization and generalization capabilities of DNNs~\cite{Du2018GradientDP,Chizat2018OnLT,Arjevani2022AnnihilationOS}. Going further, this theorem provides theoretical insights into how over-parameterization influences a model's non-convex optimization capability. The experimental results in Figure~\ref{fig:lay_init_indicator} also validate this conclusion.

\section{Empirical Validation}
\label{sec:experiments}

This section presents empirical validation of the core theoretical conclusions established in this paper through a series of targeted experiments.

\subsection{Experimental Configuration}

As illustrated in Figure~\ref{fig:paper_outline}, Theorems~\ref{thm:st_H_eigenvalue_bound} and \ref{thm:general_with_class_sgd_convergence}, along with Proposition~\ref{prop:skip_con} and Theorem~\ref{thm:number_para}, collectively form the theoretical framework for understanding non-convex optimization mechanisms in deep learning. To validate these theoretical results empirically, we design experiments aimed at assessing the alignment between our derived bounds and observed training dynamics.

The primary objective of the experiments is to evaluate whether models with different architectures exhibit behavior consistent with the theoretical conclusions during the training process. Rather than focusing on final performance metrics or comparisons with existing algorithms, we use the MNIST dataset~\cite{LeCun1998GradientbasedLA} as a baseline case study and design model architectures tailored to the specific properties being verified. To further assess the data-independence of our theoretical findings, additional experiments are conducted using the standard ResNet18 architecture~\cite{He2015DeepRL} on CIFAR-10~\cite{Krizhevsky2009LearningML} and Fashion-MNIST~\cite{Xiao2017FashionMNISTAN} datasets.

Model architectures and their configuration parameters are summarized in Figure~\ref{fig:model_increase}. Specifically, Models b and d extend Models a and c by introducing skip connections. The parameter $k$ denotes the number of blocks (or layers), enabling control over model depth.

All models are trained using the SGD optimizer with a learning rate of 0.01 and momentum of 0.9. Experiments were implemented in Python 3.7 using PyTorch 2.2.2, and executed on a GeForce RTX 2080 Ti GPU.

\subsection{Verification of Optimization Mechanism}

This subsection presents experimental designs aimed at validating the optimization mechanism described in Theorem~\ref{thm:st_H_eigenvalue_bound}. According to the theorem, minimizing both the local gradient norm and the structural error allows for effective control of the upper and lower bounds of the loss function, leading to successful optimization. Additionally, we verify that SGD indeed reduces the local gradient norm, as predicted by Theorem~\ref{thm:general_with_class_sgd_convergence}.

\subsubsection{Experimental Design}

In this experiment, we utilize the Softmax CrossEntropy loss function with a specified class count of $k=1$ and a batch size of 64. 
Based on Theorem~\ref{thm:st_H_eigenvalue_bound}, under the use of the CrossEntropy loss function, $\mathcal{L}_*(s)$ and $\mathcal{L}_*(z)$ both equal zero. Consequently, the local gradient norm and structural error serve as upper and lower bounds for $\mathcal{L}(\theta,s)$ and $\mathcal{L}(\theta,z)$, respectively.

To verify whether SGD optimizes the loss by reducing the gradient norm and controlling these bounds, we track the following quantities throughout training:
\begin{itemize}
    \item $\log \ell(f_\theta(x), y)$: the logarithmic loss,
    \item $U(A_x) + \log \|\nabla_\theta \ell(f_\theta(x), y)\|_2^2$: the upper bound,
    \item $L(A_x) + \log \|\nabla_\theta \ell(f_\theta(x), y)\|_2^2$: the lower bound,
    \item $\log \|\mathrm{Softmax}(f_\theta(x)) - y\|_2^2$: squared error in the output probability distribution.
\end{itemize}
If the upper and lower bounds decrease and converge toward zero during training, it would substantiate Theorem~\ref{thm:st_H_eigenvalue_bound}, confirming that minimizing the loss is effectively achieved by jointly reducing the gradient norm and structural error.

To quantify the evolving relationships among these quantities during training, we compute the local Pearson correlation coefficient using a sliding window approach. Given a window of length $m$, the local Pearson correlation coefficient between two variables $X$ and $Y$ is defined as:

\begin{equation}
    r_w = \frac{\sum_{i=1}^m (X_i - \bar{X})(Y_i - \bar{Y})}{\sqrt{\sum_{i=1}^m (X_i - \bar{X})^2} \sqrt{\sum_{i=1}^m (Y_i - \bar{Y})^2}},
\end{equation}
where $\bar{X} = \frac{1}{m} \sum_{i=1}^m X_i$ and $\bar{Y} = \frac{1}{m} \sum_{i=1}^m Y_i$. The value of $r_w$ ranges from $-1$ to $1$, indicating the strength and direction of the linear relationship between $X$ and $Y$: $r_w = 1$ implies perfect positive correlation, $r_w = -1$ implies perfect negative correlation, and $r_w = 0$ indicates no linear correlation. In our analysis, we employ a sliding window of length 50 to compute the correlation between the loss and its bounds.

\begin{figure}[ht]
\begin{center}
\centerline{\includegraphics[width=\columnwidth]{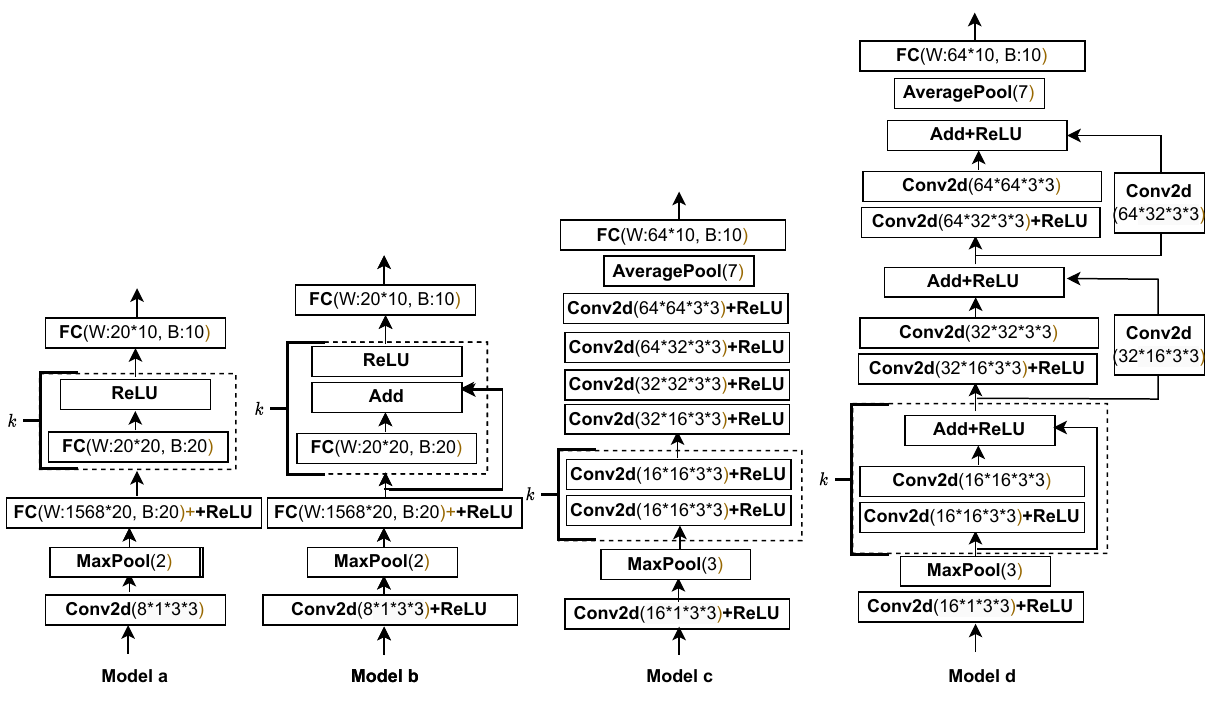}}
\caption{Model architectures and configuration parameters.}
\label{fig:model_increase}
\end{center}
\vskip -0.2in
\end{figure}

\subsubsection{Experimental Result}

Figure~\ref{fig:convergence_indicator} illustrates the evolution of gradient norm and structural error during training, while Figure~\ref{fig:convergence_bound} plots the corresponding changes in the loss and its bounds. The dynamic evolution of the local Pearson correlation coefficients is shown in Figure~\ref{fig:model_correlation} and Figure~\ref{fig:data_impact}.

\begin{figure}[ht]
	\centering
	\begin{minipage}{1.0\linewidth}
		\centering
            \centerline{\includegraphics[width=\columnwidth]{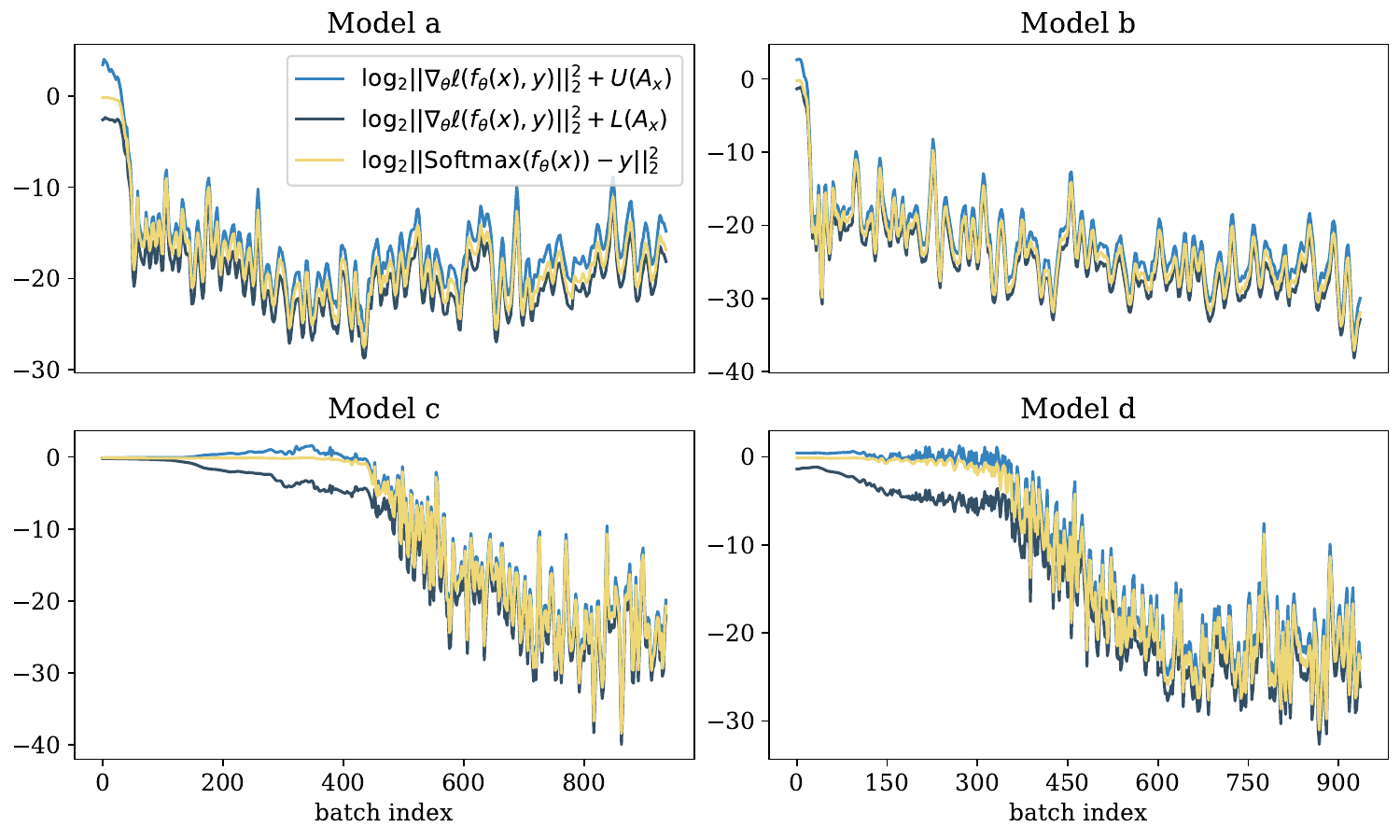}}
            \caption{Changes of bounds during the training process.}
            \label{fig:convergence_bound}
            \vspace{5mm} 
	\end{minipage}
	\begin{minipage}{1.0\linewidth}
		\centering
            \centerline{\includegraphics[width=1.0\columnwidth]{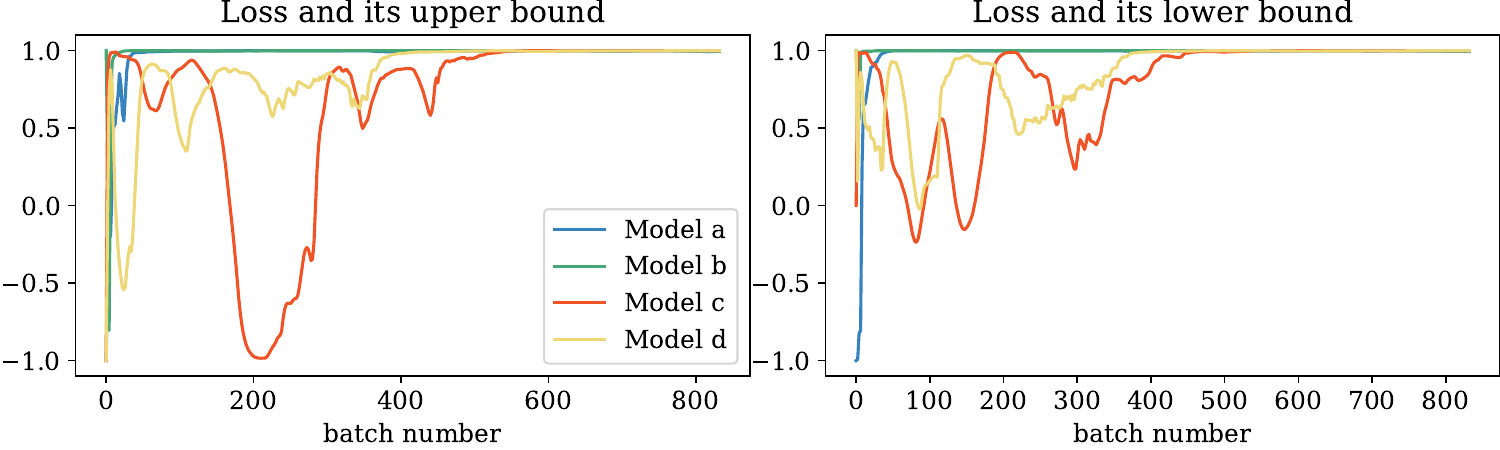}}
            \caption{Local Pearson correlation coefficient curves during the training process.}
            \label{fig:model_correlation}
	\end{minipage}
\vskip -0.1in
\end{figure}

\begin{figure}[ht]
\centering
    \centerline{\includegraphics[width=1.0\linewidth]{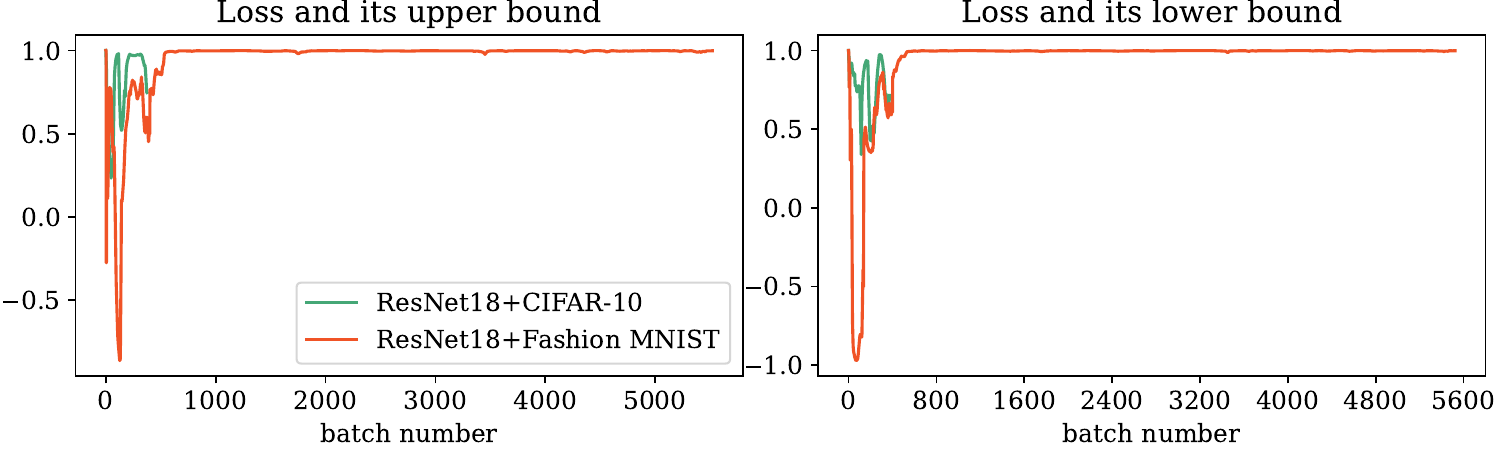}}
    \caption{Local Pearson correlation coefficient curves for different datasets.}
    \label{fig:data_impact}
\end{figure}

Key observations from Figures~\ref{fig:convergence_bound} and \ref{fig:model_correlation} include:

\begin{enumerate}
    \item As training progresses, the variation in the loss becomes increasingly aligned with the trends of its upper and lower bounds. This suggests that the theoretical bounds accurately reflect the behavior of the actual loss function during optimization.
    
    \item The gap between the upper and lower bounds narrows significantly over time, indicating reduced uncertainty or variability in the loss as training proceeds.
    
    \item By the 500th iteration, the Pearson correlation coefficients between $\log \|\mathrm{Softmax}(f_{\theta}(x))-y\|_2^2$ and its bounds approach 1 across all models. This strong positive correlation confirms the tightness of the theoretical bounds and validates their utility in characterizing the loss dynamics.
\end{enumerate}

These results provide empirical support for Theorem~\ref{thm:st_H_eigenvalue_bound}, demonstrating that minimizing loss can be effectively achieved by controlling its theoretical upper and lower bounds.

\begin{figure}[ht]
\centering
    \centerline{\includegraphics[width=1.0\linewidth]{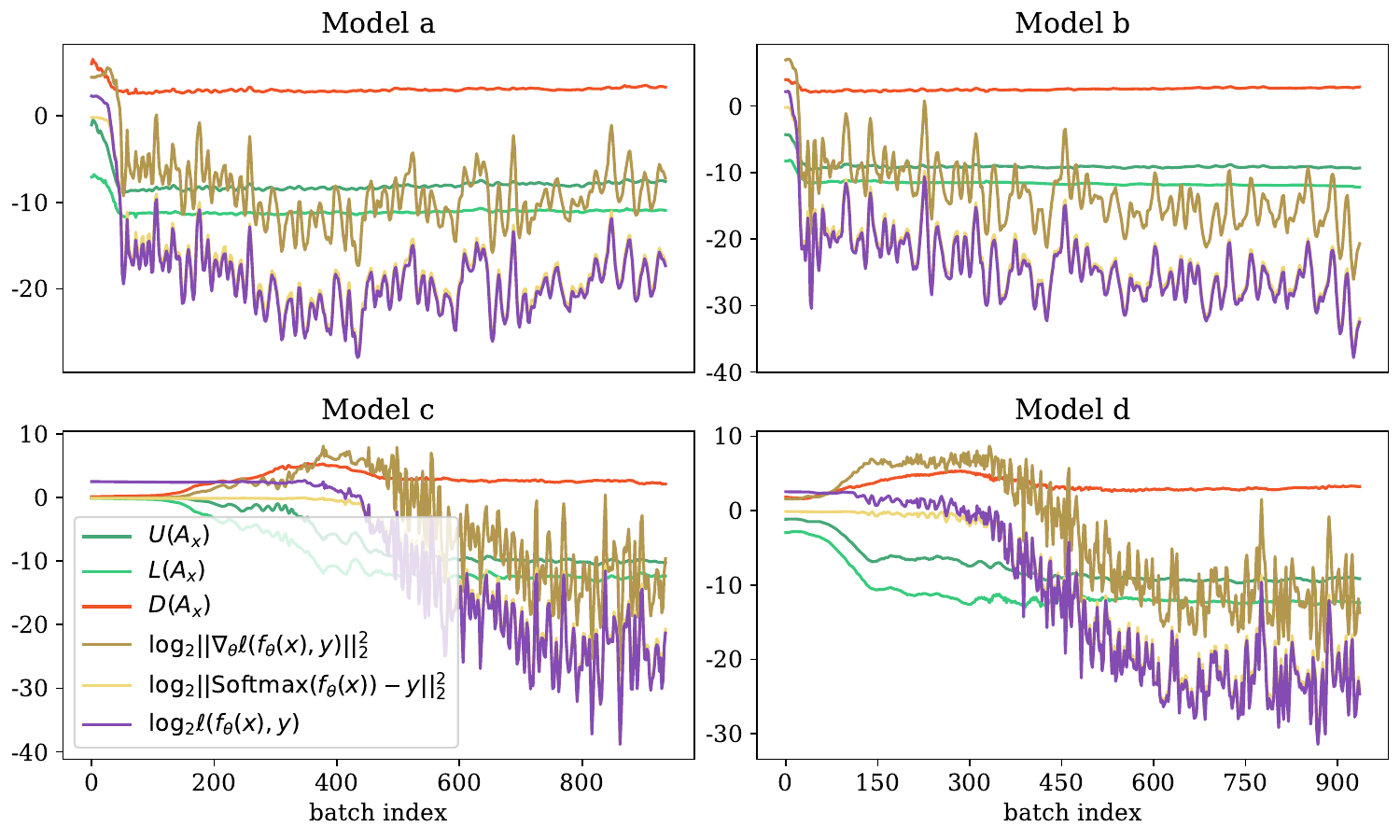}}
    \caption{Changes in indicators during the training process.}
    \label{fig:convergence_indicator}
\end{figure}
Additionally, Figure~\ref{fig:convergence_indicator} shows the changes in structural error and local gradient norm throughout the training process. These results confirm the predictions of Theorem~\ref{thm:general_with_class_sgd_convergence}, demonstrating that Mini-batch SGD successfully reduces the local gradient norm, thereby facilitating effective optimization of the non-convex objective.

Finally, Figure~\ref{fig:data_impact} shows that when using the standard ResNet18 and other datasets, the local Pearson correlation coefficients of $\log \|\mathrm{Softmax}(f_\theta(x))- y\|_2^2$ with its upper and lower bounds also gradually converge to 1. This indicates that the above conclusions are independent of the dataset.

\subsection{Verfication on Structural Error Minimization}

To empirically validate the theoretical findings in Proposition~\ref{prop:skip_con} and Theorem~\ref{thm:number_para}, we analyze changes in structural error during both network initialization and training phases. This approach allows us to empirically verify that skip connections and increasing the number of parameters under the GIC can indeed reduce structural error.

\subsubsection{Experimental Result on Initialization}

We evaluate the behavior of structural error as the depth of the models increases under standard He initialization. The resulting trends are shown in Figure~\ref{fig:lay_init_indicator}. Our observations are as follows:

\begin{enumerate}
    \item For Models a and c, networks without skip connections, as the number of layers increases, the upper bound $U(A_x)$, lower bound $L(A_x)$, and gradient norm $D(A_x)$ all decrease rapidly. Consequently, the structural error diminishes and approaches zero. This result aligns precisely with the implications of Theorem~\ref{thm:number_para}, which predicts that increasing model depth under GIC leads to reduced structural error.
    
    \item In contrast, Models b and d incorporate skip connections, which amplify gradient magnitudes and mitigate the vanishing gradient problem. As a result, while $U(A_x)$ and $L(A_x)$ still decrease with depth, $D(A_x)$ increases. Given that $U(A_x) \gg L(A_x)$, the dominant factor in structural error remains $U(A_x)$, leading to an overall reduction in structural error. These results confirm the role of skip connections in reducing structural error as described in Proposition~\ref{prop:skip_con}. Furthermore, they suggest that skip connections violate the GIC, thereby invalidating the assumptions of Theorem~\ref{thm:number_para}.
\end{enumerate}

\begin{figure}[ht]  
\centering  
\centerline{\includegraphics[width=\columnwidth]{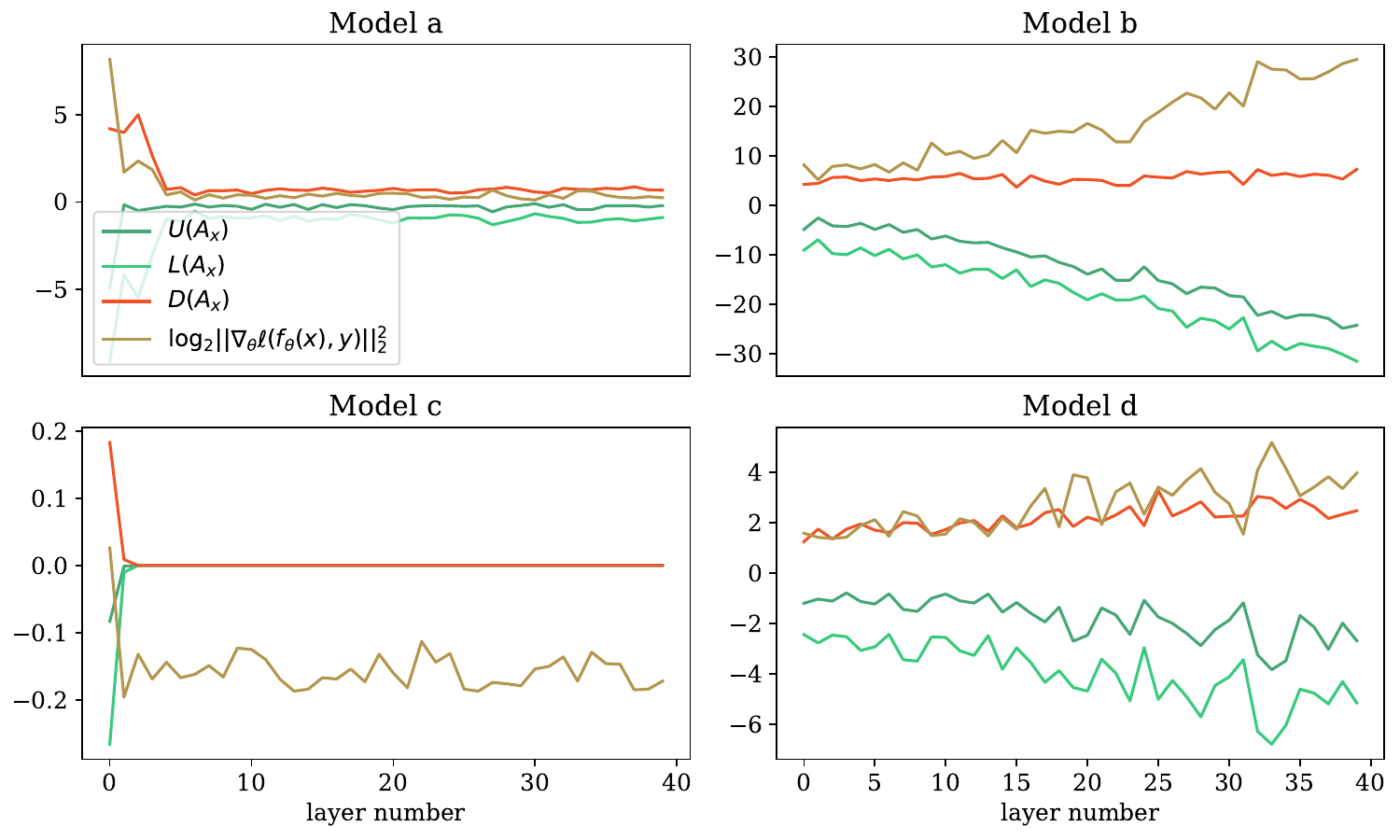}}  
\caption{Changes in indicators during the increase in model depth.}  
\label{fig:lay_init_indicator}  
\vskip -0.1in  
\end{figure}

\subsubsection{Experimental Result on Training Dynamics}

To investigate how parameter count affects training dynamics, we vary the block count $k$ from 0 to 5 in Models a and b and train each configuration. The evolution of loss bounds is depicted in Figures~\ref{fig:layer_nores_bound} and \ref{fig:layer_res_bound}, while local Pearson correlation coefficients between loss and its bounds, computed using a sliding window of length 50, are shown in Figures~\ref{fig:nores_layer_correlation} and \ref{fig:res_layer_correlation}.

\begin{figure}[htbp]
	\centering
	\begin{minipage}{1.0\linewidth}
		\centering
		\includegraphics[width=1.0\columnwidth]{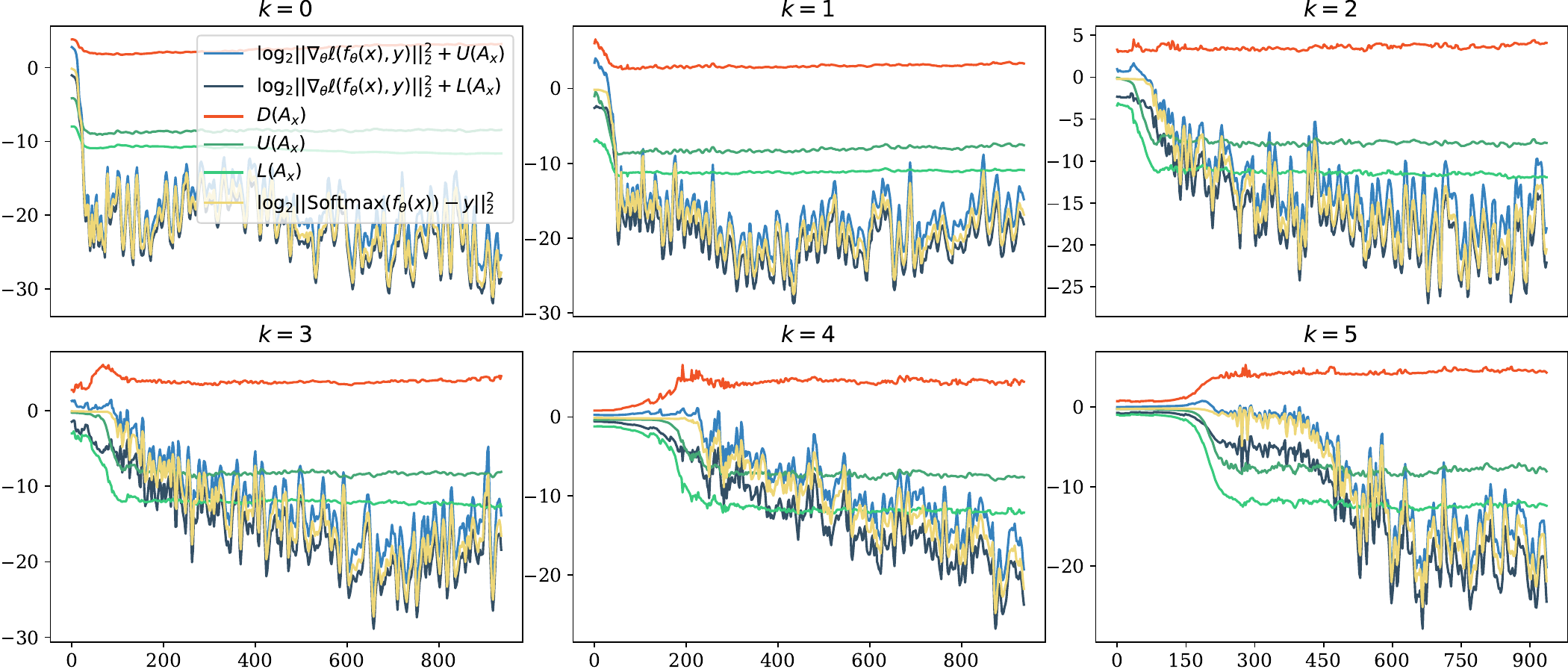}
		\caption{Bounds of model a with increasing blocks.}
		\label{fig:layer_nores_bound}
	\end{minipage}

	\begin{minipage}{1.0\linewidth}
		\centering
		\includegraphics[width=1.0\columnwidth]{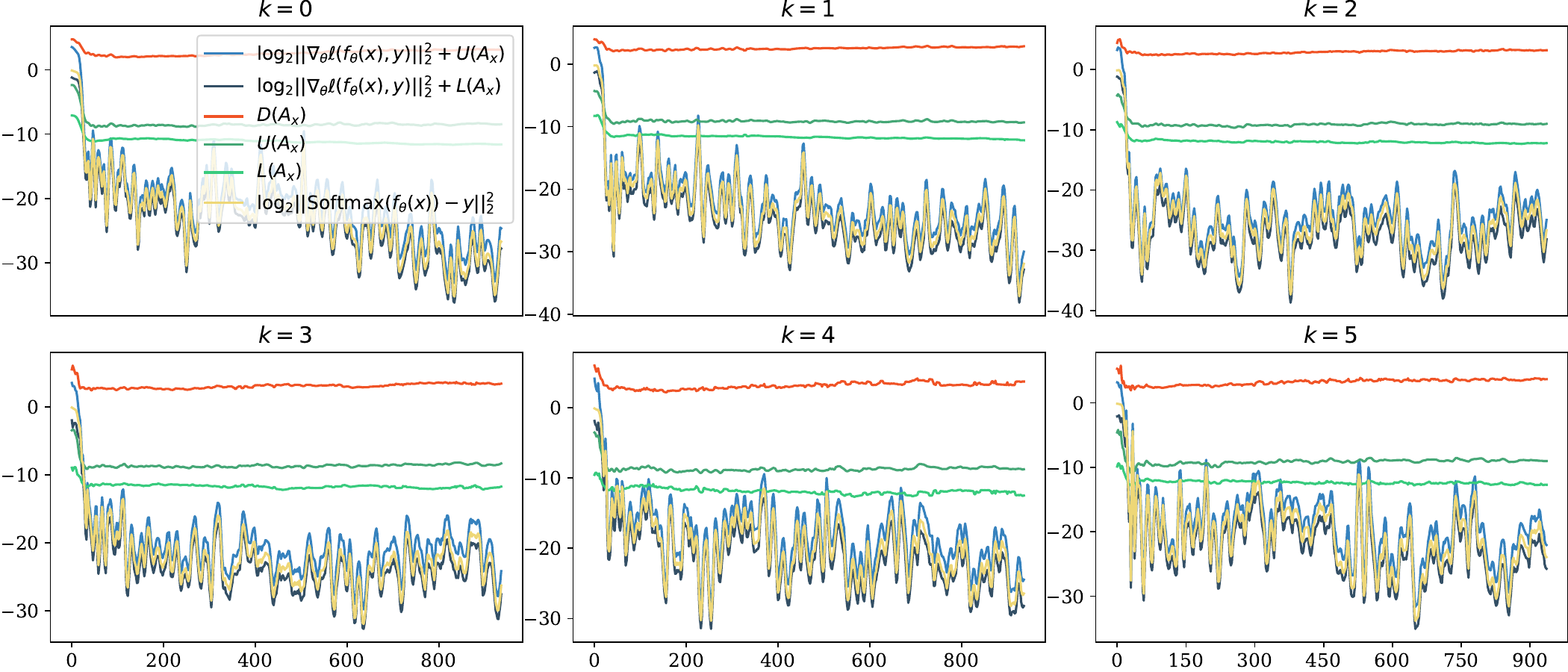}
		\caption{Bounds of model b with increasing blocks.}
		\label{fig:layer_res_bound}
	\end{minipage}
\end{figure}

\setlength{\parskip}{0.2cm plus4mm minus3mm}
\begin{figure}[htbp]
	\centering
	\begin{minipage}{0.9\linewidth}
		\centering
		\includegraphics[width=1.0\columnwidth]{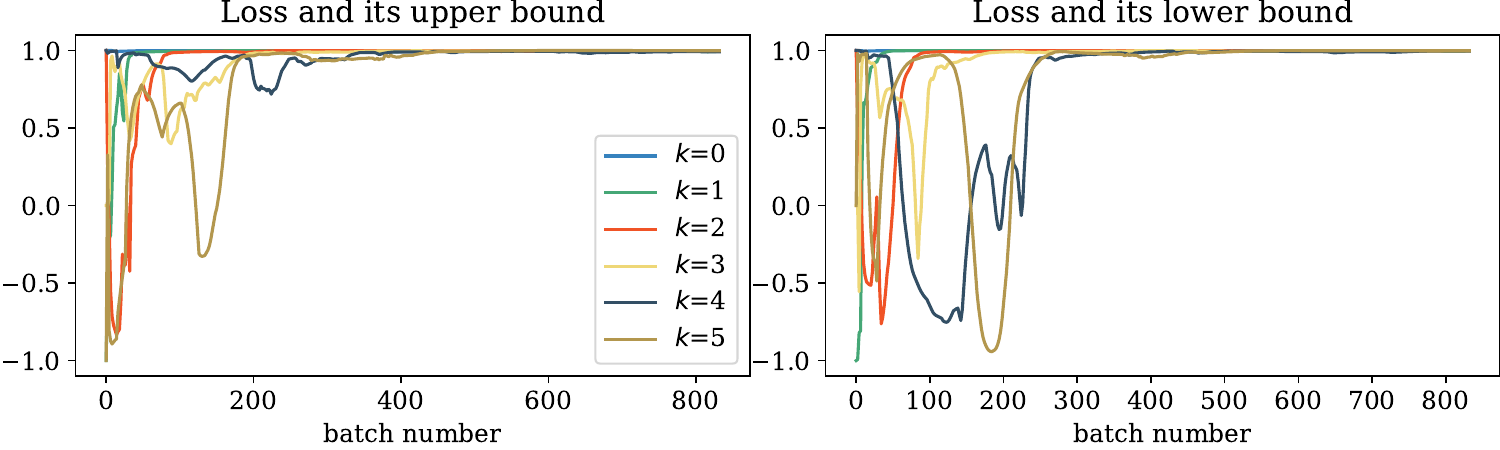}
		\caption{Local Pearson correlation coefficient curves of model a with increasing blocks.}
		\label{fig:nores_layer_correlation}
	\end{minipage}
	\begin{minipage}{0.9\linewidth}
		\centering
		\includegraphics[width=1.0\columnwidth]{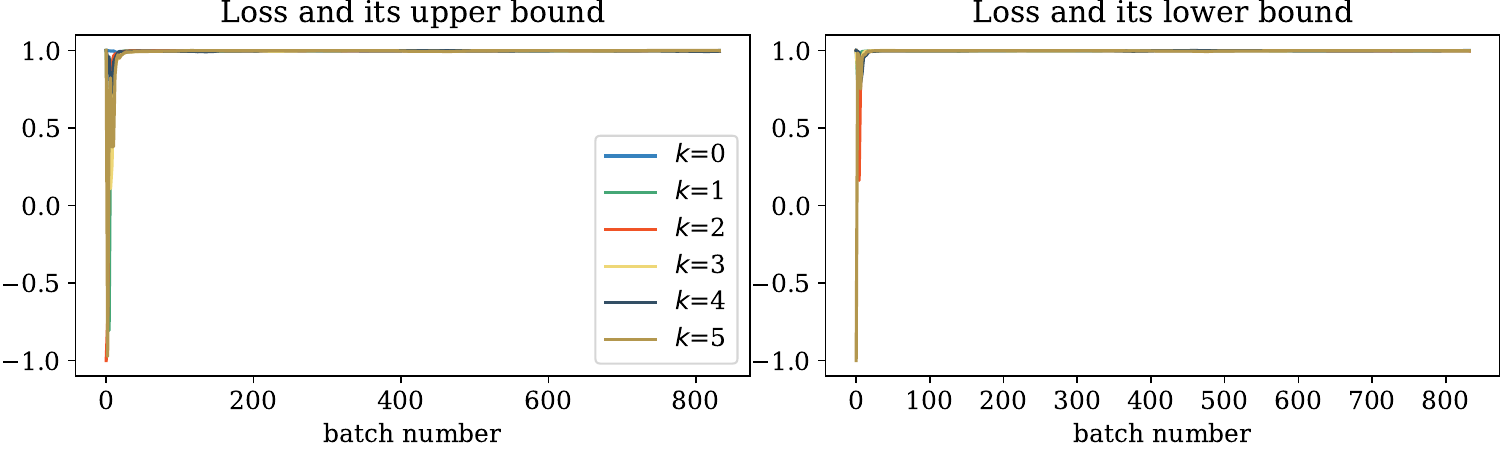}
		\caption{Local Pearson correlation coefficient curves of model b with increasing blocks.}
		\label{fig:res_layer_correlation}
	\end{minipage}
\end{figure}

\begin{figure}[!ht]
\centering
\centerline{\includegraphics[width=\columnwidth]{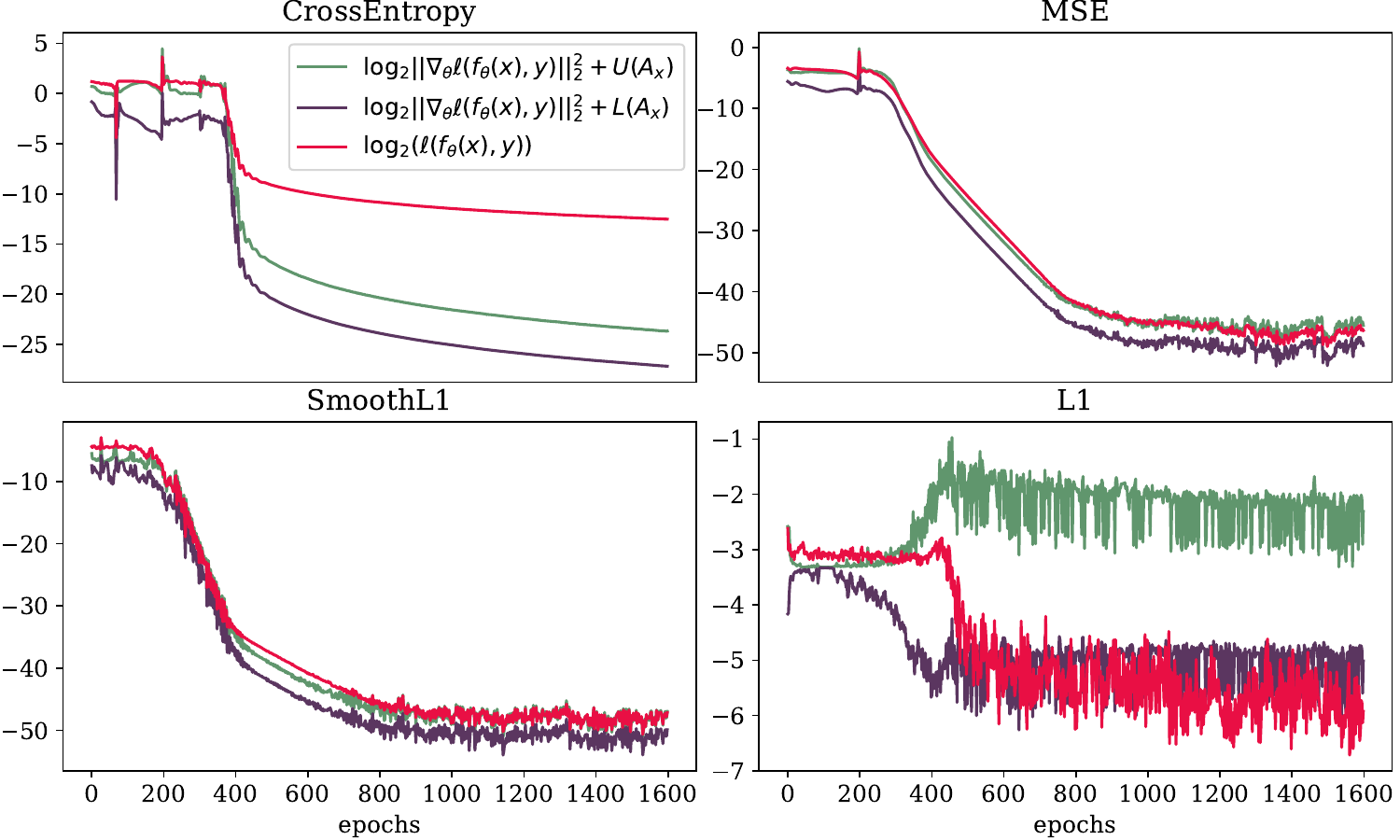}}
\caption{Analysis of $U(A_x)$, $L(A_x)$, local gradient norm, and optimization objective for the built-in loss functions.}
\label{fig:native_loss_convergence}
\end{figure}

\begin{figure}[!ht]
\centering
\centerline{\includegraphics[width=\columnwidth]{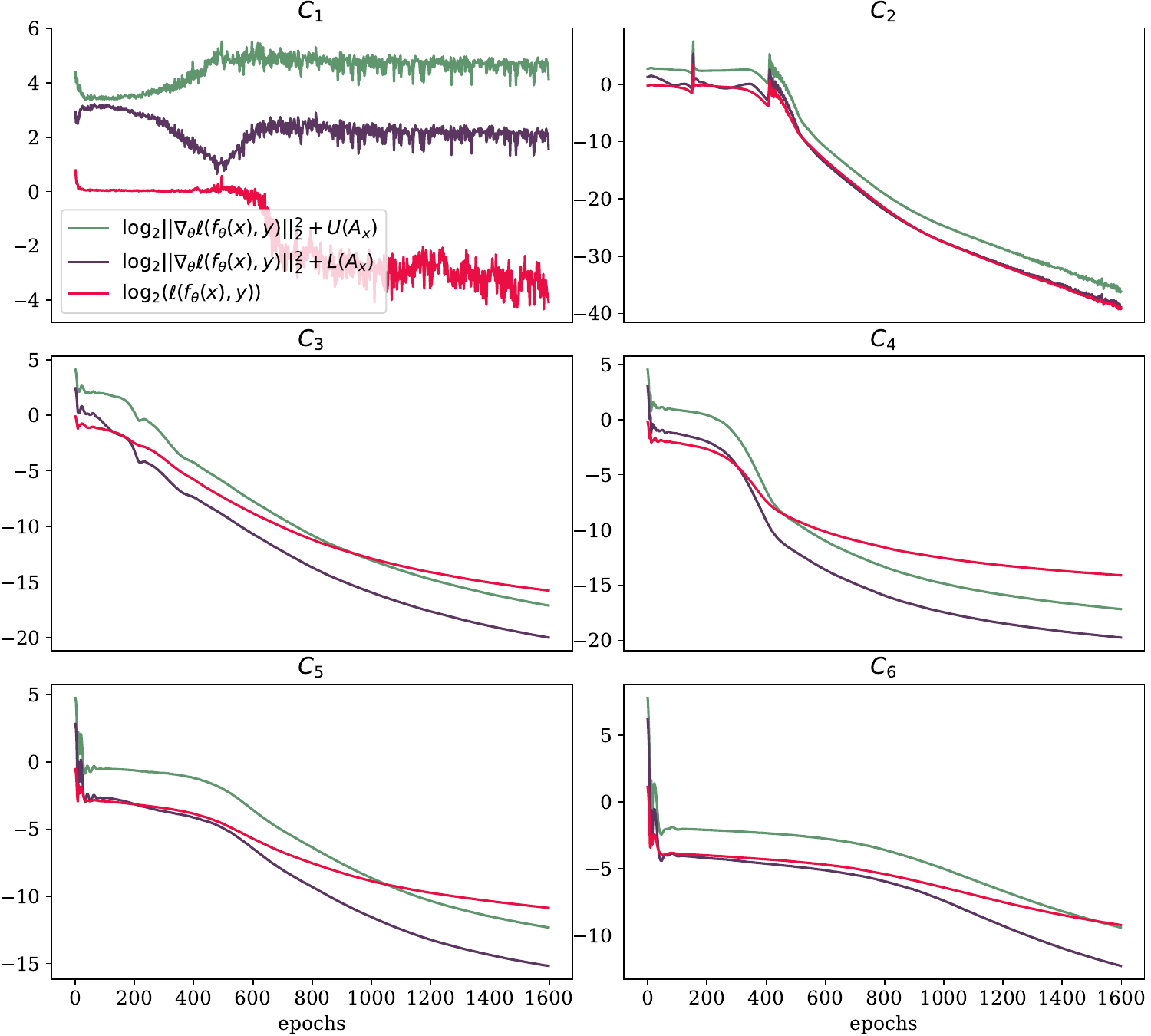}}
\caption{Analysis of $U(A_x)$, $L(A_x)$, local gradient norm, and optimization objective for the customized loss functions.}
\label{fig:custom_loss_convergence}
\end{figure}

Based on these results, we derive the following conclusions:

\begin{itemize}
    \item \textbf{Without skip connections, increasing model depth prolongs the convergence time of structural error and delays loss reduction.} As shown in Figure~\ref{fig:layer_nores_bound}, for Model a, the onset of loss reduction corresponds closely to the point at which structural error converges. Increasing depth extends this convergence period, delaying effective optimization. This observation supports Theorem~\ref{thm:st_H_eigenvalue_bound}, which posits that loss minimization depends on controlling its upper and lower bounds.

    \item \textbf{Skip connections break the initial GIC but accelerate model's convergence.} At the beginning of training, even when increasing the number of layers, Model b does not show a significant reduction in structural error early on. We attribute this to the disruption of GIC caused by skip connections, rendering Theorem~\ref{thm:number_para} inapplicable. However, the presence of skip connections addresses the vanishing gradient problem and reduces the structural error, thereby accelerating the model's convergence.
\end{itemize}

To ensure robustness of our findings and verify that the observed phenomena are not specific to the CrossEntropy loss, we conducted additional experiments using a variety of loss functions. These include built-in PyTorch losses such as MSE loss, Softmax CrossEntropy loss, L1 loss, and SmoothL1 loss, as well as a custom-defined family of losses denoted $\|f_{\theta}(x) - y\|_k^k$, where $k$ ranges from 1 to 6.

For this analysis, we selected Model d from Figure~\ref{fig:model_increase} with $k = 1$. All models were trained for 1600 epochs using SGD with consistent hyperparameters.
The convergence patterns of the loss functions and their corresponding bounds are illustrated in Figures~\ref{fig:native_loss_convergence} and \ref{fig:custom_loss_convergence}. Across all tested loss types, the loss decreases closely follow the trends of their theoretical upper and lower bounds. This consistency strongly indicates that the optimization mechanisms discussed in this paper are general principles applicable across a broad range of loss formulations.

\section{Conclusion}
\label{sec:conclusion}

In this paper, we introduce and investigate the notions of $\mathcal{H}(\phi, c_\phi)$-convexity and $\mathcal{H}(\Phi, c_\Phi)$-smoothness, which extend classical concepts such as Lipschitz smoothness and strong convexity. 
These formulations enable a more flexible and realistic characterization of loss functions and models in deep learning settings.
By leveraging these generalized properties, we establish practical constraints on both the loss functions and model architectures within the framework of empirical risk minimization, ensuring alignment with real-world deep learning applications. We then show theoretically that effective optimization in such non-convex settings can be achieved by jointly reducing the local gradient norm and the structural error.
Moreover, we demonstrate that SGD plays a crucial role in reducing the local gradient norm, while architectural and training techniques,such as over-parameterization, skip connections, and random initialization, contribute significantly to minimizing the structural error. Finally, our core theoretical findings are substantiated through comprehensive empirical validation.
We believe that the insights presented in this work not only deepen the current understanding of non-convex optimization in deep learning but also open new avenues for the development of more effective training methodologies and theoretical analyses in this rapidly evolving field.

\acks{This work was supported in part by the National Key Research and Development Program of China under No. 2022YFA1004700, Shanghai Municipal Science and Technology, China Major Project under grant 2021SHZDZX0100, the National Natural Science Foundation of China under Grant Nos.~62403360, 72171172, 92367101, 62088101, iF open-funds from Xinghuo Eco and China Institute of Communications.
}

\section*{Declarations}

\subsection{Funding}
This work was supported in part by the National Key Research and Development Program of China under No. 2022YFA1004700, Shanghai Municipal Science and Technology, China Major Project under grant 2021SHZDZX0100, the National Natural Science Foundation of China under Grant Nos.~62403360, 72171172, 92367101, 62088101, iF open-funds from Xinghuo Eco and China Institute of Communications.

\subsection{Competing interests}
The authors declare that they have no known competing financial interests or personal relationships that could have appeared to influence the work reported in this paper.

\subsection{Ethics approval and consent to participate}
Not applicable.

\subsection{Data availability}
The datasets generated and/or analyzed in the present study are accessible via the GitHub repository at \href{https://yann.lecun.com/exdb/mnist/}{minist} and the website at \href{https://yann.lecun.com/exdb/mnist/}{minist}.

\subsection{Materials availability}
Not applicable.

\subsection{Code availability}
The source code will be provided after acceptance for reproducibility purposes.

\subsection{Author contribution}
Binchuan Qi: Conceptualization, Methodology, Writing original draft. 
Wei Gong: Supervision, Review, Declarations. 
Li Li: Supervision, Review, Declarations.

\appendix

\section{Other Related Works}
\label{appendix:other_related_work}

\subsection{Gradient-based Optimization} 
The classical gradient-based optimization problems concerning standard Lipschitz smooth functions have been extensively studied for both convex~\cite{Darzentas1983ProblemCA,Nesterov2014IntroductoryLO} and non-convex functions. 
In the convex scenario, the objective is to identify an $\varepsilon$-sub-optimal point $x$ such that $f(x)-\inf_x f(x)\le \varepsilon$. For convex Lipschitz smooth problems, it has been established that GD attains a gradient complexity of $\mathcal{O}(1/\varepsilon)$. When dealing with strongly convex functions, GD achieves a complexity of $\mathcal{O}(\kappa\log (1/\varepsilon))$, where $\kappa$ denotes the condition number.
In the non-convex domain, the aim is to locate an $\varepsilon$-stationary point $x$ satisfying $\|\nabla f(x)\|\le \varepsilon$, given that identifying a global minimum is generally NP-hard. For deterministic non-convex Lipschitz smooth problems, it is well-established that GD reaches the optimal complexity of $\mathcal{O}(1/\varepsilon^2)$, matching the lower bound presented in the paper~\cite{Carmon2017LowerBF}.
Within the stochastic context, for an unbiased stochastic gradient with bounded variance, SGD achieves the optimal complexity of $\mathcal{O}(1/\varepsilon^4)$~\cite{Ghadimi2013StochasticFA,Arjevani2019LowerBF}. Fang et al.~\cite{Fang2018SPIDERNN} introduced the first variance reduction algorithm, SPIDER, which attains the optimal sample complexity of $\mathcal{O}(\varepsilon^{-3})$ under the stronger expected smoothness assumption. Concurrently, several other variance reduction algorithms have been devised for stochastic non-convex optimization, achieving the optimal sample complexity, such as SARAH~\cite{Nguyen2017SARAHAN}, SpiderBoost~\cite{Wang2019SpiderBoostAM}, STORM~\cite{Cutkosky2019MomentumBasedVR}, and SNVRG\cite{Zhou2018StochasticNV}.
For objective functions exhibiting $(L_0,L_1)$-smoothness, Zhang et al.~\cite{Zhang2019WhyGC} also put forth clipped GD and normalized GD, which maintain the optimal iteration complexity of $\mathcal{O}(1/\varepsilon^{2})$, and proposed clipped SGD, which likewise achieves a sample complexity of $\mathcal{O}(1/\varepsilon^{4})$. Zhang et al.~\cite{Zhang2020ImprovedAO} developed a general framework for clipped GD/SGD with momentum acceleration, achieving identical complexities for both deterministic and stochastic optimization. Reisizadeh et al.~\cite{Reisizadeh2023VariancereducedCF} managed to reduce the sample complexity to $\mathcal{O}(1/\varepsilon^{3})$ by integrating the SPIDER variance reduction technique with gradient clipping. For objective functions characterized by $\ell$-smoothness, Li et al.~\cite{Li2023ConvexAN} achieves the same complexity for GD/SGD as under the classical Lipschitz smoothness condition.

\subsection{KŁ Theory}
In the context of non-convex and non-smooth scenarios, the KŁ property is of significant importance and has been extensively studied in variational analysis and optimization.  
For an extended-real-valued function $f$ and a point $\bar{x}$ where $f$ is finite and locally lower semicontinuous, the KL property of $f$ at $\bar{x}$ is characterized by the existence of some $\epsilon > 0$ and $\nu \in (0, +\infty]$ such that the KŁ inequality
\begin{equation}\label{eq:kl_property}
    \psi'(f(x) - f(\bar{x})) d(0, \partial f(x)) \geq 1
\end{equation}
holds for all $x$ satisfying $\|x - \bar{x}\| \leq \epsilon$ and $f(\bar{x}) < f(x) < f(\bar{x}) + \nu$. Here, $d(0, \partial f(x))$ denotes the distance of $0$ from the limiting subdifferential $\partial f(x)$ of $f$ at $x$, and $\psi: [0, \nu) \to \mathbb{R}_+$, referred to as a desingularizing function, is assumed to be continuous, concave, and of class $\mathcal{C}^1$ on $(0, \nu)$, with $\psi(0) = 0$ and $\psi' > 0$ over $(0, \nu)$.
The foundational work on the KŁ property originated from Łojasiewicz \cite{lojasiewicz1963propriete} and Kurdyka \cite{kurdyka1998gradients}, who initially studied differentiable functions. 
Subsequently, the KŁ property was extended to nonsmooth functions in the work~\cite{bolte2007lojasiewicz,bolte2007clarke}. 
By employing non-smooth desingularizing functions, a generalized version of the concave KŁ property, along with its exact modulus, was introduced and studied by \cite{wang2022exact}. Among the widely used desingularizing functions is $\psi(s) := \frac{1}{\mu} s^{1-\theta}$, where $\theta \in [0, 1)$ is the exponent and $\mu > 0$ is the modulus. In this case, the corresponding KŁ inequality can be expressed as:
\begin{equation}
    d(0, \partial f(x)) \geq \frac{\mu}{1-\theta} (f(x) - f(\bar{x}))^{\theta}.
\end{equation}
The KŁ property, along with its associated exponent (especially the case when $\theta = 1/2$) and modulus, plays a pivotal role in estimating the local convergence rates of many first-order optimization algorithms. This connection has been extensively explored in the literature, including works such as \cite{attouch2009convergence,attouch2010proximal,attouch2013convergence,li2016douglas,li2018calculus,li2023variational}, and references therein. 
Currently, there has been some research on the use of the KŁ property within a stochastic framework, and significant progress has been made \cite{Gadat2017OPTIMAL,driggs2021springfaststochasticproximal,NEURIPS2022_65ae674d,2023CONVERGENCE,li2023convergencerandomreshufflingkurdykalojasiewicz,fest2024stochasticusekurdykalojasiewiczproperty}. These studies explore how the KŁ property can be leveraged to analyze the stochastic, nonconvex optimization problems, providing valuable insights into the optimization mechanisms of such problems. 
However, directly applying the KŁ property to analyze the optimization mechanisms specific to deep learning remains rare. 

The primary focus of this paper is to theoretically analyze the optimization mechanisms specific to deep learning. While the application and improvement of the KŁ property are not within the scope of this study, the proposed $\mathcal{H}(\phi)$-convexity shares similarities with the KŁ property under the case where the desingularizing function is given by $\psi(s):=\frac{1}{\mu}s^{1-\theta}$. The analysis of the similarities and differences between them are provided in Subsection~\ref{subsec:extend_cs}.

\section{Mathematical Background}
\label{appendix:math_background}

\subsection{Norm}

Let $\mathbb{S}$, $\mathbb{S}_{+}$, and $\mathbb{S}_{++}$ denote the sets of symmetric, symmetric positive semidefinite, and symmetric positive definite $n \times n$ matrices, respectively.

A function $f: \mathbb{R}^n \rightarrow \mathbb{R}$ is a norm if it satisfies the following three properties:
\begin{enumerate}
    \item positive definiteness: $f(x)>0, \forall x \neq 0$, and $f(0)=0$.
    \item 1-homogeneity: $f(\lambda x)=|\lambda| f(x), \forall x \in \mathbb{R}^n, \forall \lambda \in \mathbb{R}$.
    \item triangle inequality: $f(x+y) \leq f(x)+f(y), \forall x, y \in \mathbb{R}^n$.
\end{enumerate}

Some well-known examples of norms include the 1-norm, $f(x)=\sum_{i=1}^n\left|x_i\right|$, the 2-norm, $f(x)=\sqrt{\sum_{i=1}^n x_i^2}$, and the $\infty$-norm, $f(x)=\max _i\left|x_i\right|$. These norms are part of a family parametrized by a constant traditionally denoted $p$, with $p \geq 1$ : the $L_p$-norm is defined by
$$
\|x\|_p=\left(\left|x_1\right|^p+\cdots+\left|x_n\right|^p\right)^{1 / p}
$$
Another important family of norms are the quadratic norms. For $P \in \mathbb{S}_{++}^n$, we define the $P$-quadratic norm as
$$
\|x\|_P=\left(x^T P x\right)^{1 / 2}=\left\|P^{1 / 2} x\right\|_2
$$

The unit ball of a quadratic norm is an ellipsoid (and conversely, if the unit ball of a norm is an ellipsoid, the norm is a quadratic norm).

\begin{lemma}
    \label{lem:norm_eq}
    Suppose that $\|\cdot\|_{\mathrm{a}}$ and $\|\cdot\|_{\mathrm{b}}$ are norms on $\mathbb{R}^n$. A basic result of analysis is that there exist positive constants $\alpha$ and $\beta$ such that, for all $x \in \mathbb{R}^n$~\cite{boyd2004convex},
$$
\alpha\|x\|_{\mathrm{a}} \leq\|x\|_{\mathrm{b}} \leq \beta\|x\|_{\mathrm{a}} .
$$
\end{lemma}
We conclude that any norms on any finite-dimensional vector space are equivalent, but on infinite-dimensional vector spaces, the result need not hold.

\subsection{Legendre-Fenchel Conjugate}
\begin{definition}[Legendre-Fenchel conjugate]
    The Legendre-Fenchel conjugate of a function $\Omega$ is denoted by 
    \begin{equation*}
        \Omega^*(\nu):= \sup_{\mu\in \mathrm{dom}(\Omega)}\langle \mu,\nu \rangle-\Omega(\mu).
    \end{equation*}
\end{definition}
By default, $\Omega$ is a continuous strictly convex function, and its gradient with respect to $\mu$ is denoted by $\mu_{\Omega}^*$. When $\Omega(\cdot)=\frac{1}{2}\|\cdot\|_2^2$, we have $\mu=\mu_\Omega^*$. 

\begin{lemma}[Properties of conjugate duality]
\label{prop:fenchel_duality}
The following are the properties of the Legendre-Fenchel conjugate~\cite{Todd2003ConvexAA}:  
\begin{enumerate}
    \item $\Omega^*(\mu)$ is always a convex function of $k$ (independently of the shape of $\Omega$).
    \item The Fenchel-Young inequality holds:
\begin{equation}
\Omega(\mu) + \Omega^*(\nu) \geq \langle \mu, \nu \rangle \quad \forall \mu \in \mathrm{dom}(\Omega), \nu \in \mathrm{dom}(\Omega^*).
\label{eq:fenchel_young_inequality}
\end{equation}
Equality is achieved when $\nu = \mu_{\Omega}^*$.
\end{enumerate}
\end{lemma}

\subsection{Fenchel-Young Loss}
\begin{definition}[Fenchel-Young loss]
    The Fenchel-Young loss $d_\Omega \colon \mathrm{dom}(\Omega) \times \mathrm{dom}(\Omega^*) \to \mathbb{R}^+$ \label{def:FY_loss} generated by $\Omega$ is defined as:
\begin{equation}
d_{\Omega}(\mu; \nu) 
:= \Omega(\mu) + \Omega^*(\nu) - \langle \mu,\nu\rangle,
\label{eq:fy_losses}
\end{equation}
where $\Omega^*$ denotes the Legendre-Fenchel conjugate of $\Omega$.
\end{definition}

\begin{table*}[ht]
    \caption{Examples of common loss functions and their corresponding standard loss forms~\cite{Blondel2019LearningWF}}.
\begin{center}
\begin{small}
\begin{threeparttable}          
\begin{tabular}{@{\hskip 0pt}l@{\hskip 0pt}c@{\hskip 0pt}c@{\hskip 5pt}c@{\hskip 10pt}c@{\hskip 0pt}}
\hline
Loss & $\mathrm{dom}(\Omega)$ & $\Omega(\mu)$ & $\hat{y}_{\Omega}(\theta)$ & $d_{\Omega}(\theta; y)$ \\
\hline
Squared  
& $\mathbb{R}^d$ & $\frac{1}{2}\|\mu\|^2$ & $\theta$ & $\frac{1}{2}\|y-\theta\|^2$ \smallskip
\\
Perceptron  
& $\Delta^{|\mathcal{Y}|}$ & $0$ & $\arg\max(\theta)$ 
& $\max_i \theta_i - \theta_k$ 
\smallskip
\\
Logistic  
& $\Delta^{|\mathcal{Y}|}$ & $-H(\mu)$ & $\mathrm{softmax}(\theta)$ &  
$\log\sum_i\exp \theta_i - \theta_k$ 
\smallskip
\\
Hinge 
& $\Delta^{|\mathcal{Y}|}$ & $\langle {\mu},{\mathbf{e}_k - \mathbf{1}}\rangle$ & 
$\arg\max(\mathbf{1}-\mathbf{e}_k+\theta)$
& $\max_i ~ [[i \neq k]] + \theta_i - \theta_k$ 
\smallskip
\\
Sparsemax 
& $\Delta^{|\mathcal{Y}|}$ & $\frac{1}{2}\|\mu\|^2$ & $\mathrm{sparsemax}(\theta)$ & 
$\frac{1}{2}\|y-\theta\|^2 - \frac{1}{2}\|\hat{y}_{\Omega}(\theta) - \theta\|^2$ 
\smallskip
\\
Logistic (one-vs-all) 
& $[0,1]^{|\mathcal{Y}|}$ 
& $-\sum_i H([\mu_i,1-\mu_i])$ & $\mathrm{sigmoid}(\theta)$ & 
$\sum_i \log(1 + \exp(-(2 y_i-1) \theta_i))$
\smallskip
\\
\hline
\end{tabular}
\begin{tablenotes}    %
    \footnotesize               
    \item[1] $\mathbf{e}_i$ represents a standard basis ("one-hot") vector.
    \item[2] $\hat{y}_\Omega \in \arg\min_{\mu \in \mathrm{\Omega}} d_\Omega(\theta,\mu)$.
    \item[3] We denote the Shannon entropy by $H(p) := -\sum_i p_i\log p_i$, where $p \in \Delta^{\mathcal{Y}}$.
\end{tablenotes}            
\end{threeparttable}       
\end{small}
\end{center}
\label{tab:fy_losses_examples}
\end{table*}

\begin{lemma}[Properties of Fenchel-Young losses]
\label{prop:fy_losses}
The following are the properties of Fenchel-Young losses~\cite{Blondel2019LearningWF}.
\begin{enumerate}
\item $d_{\Omega}(\mu, \nu) \ge 0$ for any $\mu \in
    \mathrm{dom}(\Omega)$ and $\nu \in \mathrm{dom}(\Omega^*)$. If $\Omega$ is a lower semi-continuous proper convex function, then the loss is zero if and only if $\nu \in \partial \Omega(\mu)$. Furthermore, when $\Omega$ is strictly convex, the loss is zero if and only if $\nu=\mu_{\Omega}^*$.
\item $d_{\Omega}(\mu,\nu)$ is convex with respect to $\nu$, and its subgradients include the residual vectors: $\nu_{\Omega^*}^* - \mu \in
    \partial_\nu d_{\Omega}(\mu, \nu)$. If $\Omega$ is strictly convex, then $d_\Omega(\mu,\nu)$ is differentiable and  $\nabla_\nu d_{\Omega}(\mu,\nu) = \nu_{\Omega^*}^* - \mu$. If $\Omega$ is strongly convex, then $d_\Omega(\mu,\nu)$ is smooth, i.e., $\nabla_\nu d_\Omega(\mu,\nu)$ is Lipschitz continuous. 
\end{enumerate}
\end{lemma}

\subsection{Strong Convexity and Lipschitz Smoothness}
\begin{lemma}
    Suppose $f: \mathbb{R}^n \rightarrow \mathbb{R}$ with the extended-value extension. The following conditions are all implied by strong convexity with parameter $\mu$~\cite{zhou2018fenchel}:
\begin{enumerate}
    \item  $\frac{1}{2}\left\|s_x\right\|^2 \geq \mu\left(f(x)-f^*\right), \forall x$ and $s_x \in \partial f(x)$.
    \item $\left\|s_y-s_x\right\| \geq \mu\|y-x\|, \forall x, y$ and any $s_x \in \partial f(x), s_y \in \partial f(y)$.
    \item $f(y) \leq f(x)+s_x^T(y-x)+\frac{1}{2 \mu}\left\|s_y-s_x\right\|^2, \forall x, y$ and any $s_x \in \partial f(x), s_y \in \partial f(y)$.
    \item $\left(s_y-s_x\right)^T(y-x) \leq \frac{1}{\mu}\left\|s_y-s_x\right\|^2, \forall x, y$ and any $s_x \in \partial f(x), s_y \in \partial f(y)$.
\end{enumerate}
\end{lemma}

\begin{lemma}
For a function $f$ with a Lipschitz continuous gradient over $\mathbb{R}^n$, the following relations hold~\cite{zhou2018fenchel}:
$$
[5] \Longleftrightarrow[7] \Longrightarrow[6] \Longrightarrow[0] \Longrightarrow[1] \Longleftrightarrow[2] \Longleftrightarrow[3] \Longleftrightarrow[4]
$$
If the function $f$ is convex, then all the conditions $[0]-[7]$ are equivalent.
\par $[0] \|\nabla f(x)-\nabla f(y)\| \leq L\|x-y\|, \forall x, y$.
\par $[1] \, g(x)=\frac{L}{2} x^T x-f(x)$ is convex, $\forall x$.
\par $[2] \, f(y) \leq f(x)+\nabla f(x)^T(y-x)+\frac{L}{2}\|y-x\|^2, \forall x, y$.
\par $[3] \, \left(\nabla f(x)-\nabla f(y)\right)^T(x-y) \leq L\|x-y\|^2, \forall x, y$.
\par $[4] \, f(\alpha x+(1-\alpha) y) \geq \alpha f(x)+(1-\alpha) f(y)-\frac{\alpha(1-\alpha) L}{2}\|x-y\|^2, \forall x, y$ and $\alpha \in[0,1]$.
\par $[5] \, f(y) \geq f(x)+\nabla f(x)^T(y-x)+\frac{1}{2 L}\|\nabla f(y)-\nabla f(x)\|^2, \forall x, y$.
\par $[6] \, \left(\nabla f(x)-\nabla f(y)\right)^T(x-y) \geq \frac{1}{L}\|\nabla f(x)-\nabla f(y)\|^2, \forall x, y$.
\par $[7] \, f(\alpha x+(1-\alpha) y) \leq \alpha f(x)+(1-\alpha) f(y)-\frac{\alpha(1-\alpha)}{2 L}\|\nabla f(x)-\nabla f(y)\|^2, \forall x, y$ and $\alpha \in[0,1]$.
    
\end{lemma}

\begin{lemma}
\label{lem:convex_smooth_dual}
If $f$ is strongly convex with parameter $c$, it follows that
\begin{equation}\label{eq:strong}
    f(y) \geq f(x) + \nabla f(x)^T(y-x) + \frac{c}{2}\|y-x\|^2, \quad \forall x, y.
\end{equation}
If $f$ is Lipschitz-smooth over $\mathbb{R}^n$, then it follows that
\begin{equation}\label{eq:lipschitz}
    f(y) \leq f(x) + \nabla f(x)^T(y-x) + \frac{L}{2}\|y-x\|^2, \quad \forall x, y.
\end{equation}
For the proofs of the above results, please refer to the studies~\cite{kakade2009duality,zhou2018fenchel}.
They  further reveal the duality properties between strong convexity and Lipschitz smoothness as follows.
For a function $f$ and its Fenchel conjugate function $f^*$, the following assertions hold:
\begin{itemize}
    \item If $f$ is closed and strongly convex with parameter $c$, then $f^*$ is Lipschitz smooth with parameter $1/c$.
    \item If $f$ is convex and Lipschitz smooth with parameter $L$, then $f^*$ is strongly convex with parameter $1/L$.
\end{itemize}
\end{lemma}

\subsection{Legendre Function}
\begin{definition}[Legendre function]
\label{def:legendre_fun}
The proper, lower semicontinuous convex function $\phi: \mathbb{R}^m \to \bar{\mathbb{R}}$ is:
\begin{itemize}
    \item \textbf{Essentially smooth} if the interior of its domain is nonempty, i.e., $\mathrm{int}(\mathrm{dom}(\phi)) \neq \emptyset$, and $\phi$ is differentiable on $\mathrm{int}(\mathrm{dom}(\phi))$ with $\|\nabla\phi(w^{\nu})\| \to \infty$ as $w^{\nu} \to w \in \text{bdry } \mathrm{dom}(\phi)$;
    \item \textbf{Essentially strictly convex} if $\phi$ is strictly convex on every convex subset of $\mathrm{dom}(\partial \phi): \{w \in \mathbb{R}^m : \partial\phi(w) \neq \emptyset\}$;
    \item \textbf{Legendre} if $\phi$ is both essentially smooth and essentially strictly convex~\cite{1970Convex,Strmberg2009ANO}.
\end{itemize}
\end{definition}
\begin{lemma}\label{lem:convex_diff}
Given that $F$ is a Legendre function, $F^*$ is also of Legendre type~\cite{1970Convex,Strmberg2009ANO}.
\end{lemma}

\subsection{Lemmas}
\label{appendix:lem}
\begin{lemma}[Pinsker's inequality]~\cite{2017Elements,1967Information,Kullback1967ALB}
\label{lem:pinsker}
If $p$ and $q$ are probability densities/masses both supported on a bounded interval $I$, then we have
\begin{equation}
    D_{KL}(p\|q) \ge \frac{1}{2\ln 2} \|p-q\|_1^2.
\end{equation}
\end{lemma}
\begin{lemma}
   \label{lem:kl_upper_bound}
    If $p$ and $q$ are probability densities/masses both supported on a bounded interval $I$, then we have~\cite{1603768}
\begin{equation}
    D_{\textrm{KL}}(p,q) \leq \frac{1}{\inf_{x\in I} q(x)} \|p-q\|_2^2.
\end{equation}
\end{lemma}

\begin{lemma}[Euler's theorem for homogeneous functions]
\label{lem:euler}
If $f: \mathbb{R}^n \to \mathbb{R}$ is a differentiable $k$-homogeneous function, then
$$
x \cdot \nabla f(x) = k f(x),
$$
where $x \cdot \nabla f(x) = \sum_{i=1}^n x_i \frac{\partial f}{\partial x_i}$.
\end{lemma}

\begin{lemma}
\label{prop:high_dim_2}
    Suppose that we sample $n$ points $x^{(1)}, \ldots, x^{(n)}$ uniformly from the unit ball $\mathcal{B}^{m}_1: \{x \in \mathbb{R}^m, \|x\|_2 \le 1\}$. Then with probability $1 - O(1 / n)$ the following holds~\cite[Theorem 2.8]{blum2020foundations}:
\begin{equation}
    \begin{aligned}
                &\left\|x^{(i)}\right\|_2 \geq 1 - \frac{2 \log n}{m}, \text{ for } i = 1, \ldots, n.\\
        & |\langle x^{(i)}, x^{(j)} \rangle| \leq \frac{\sqrt{6 \log n}}{\sqrt{m-1}} \text{ for all } i, j = 1, \ldots, i \neq j.   
    \end{aligned}
\end{equation}
\end{lemma}
\begin{lemma}[Gershgorin’s circle theorem]
\label{lemma:gershgorin}
    Let $A$ be a real symmetric $n \times n$ matrix, with entries $a_{ij}$. For $i = \{1, \cdots, n\}$ let $R_i$ be the sum of the absolute value of the non-diagonal entries in the $i$-th row: $R_i = \sum_{1 \le j \le n, j \neq i} |a_{ij}|$.
    For any eigenvalue $\lambda$ of $A$, there exists $i \in \{1, \cdots, n\}$ such that 
    \begin{equation}
        |\lambda - a_{ii}| \le R_i.
    \end{equation}
\end{lemma}

\begin{lemma}
    \label{lem:eigen_bound}
    Let $A$ be an $n$-dimensional real symmetric matrix. Then, for any vector $e \in \mathbb{R}^n$, the following inequality holds:
\begin{equation}
\lambda_{\min}(A) \|e\|_2^2 \leq e^\top A e \leq \lambda_{\max}(A) \|e\|_2^2.
\end{equation}
\end{lemma}

\begin{lemma}
\label{lem:min_value}
Define $s = \frac{r}{r-1}$, $K(x) = ax^{r} - bx$, where $x, a, b \in \mathbb{R}_{>0}$, and $r > 1$. The following conclusions are established:
\begin{itemize}
    \item When $x = \left(\frac{b}{ra}\right)^{s-1}$, $K(x)$ attains its minimum value, given by
$\min_x K(x) = -s^{-1} b^{s} (ra)^{1-s}$.
    \item When $0 < x < \left(\frac{b}{a}\right)^{s-1}$, it holds that $K(x) \le 0$.
\end{itemize}
\end{lemma}

\section{Appendix for Proofs}

\label{appendix:proof}
In this section, we prove the results stated in the paper and provide necessary technical details and discussions.

\subsection{Proof of Lemma~\ref{lem:eigen_bound}}
\label{appendix:proof_eigen_bound}

\begin{lemma}

    Let $A$ be an $n$-dimensional real symmetric matrix. Then, for any vector $e \in \mathbb{R}^n$, the following inequality holds:
\begin{equation}
\lambda_{\min}(A) \|e\|_2^2 \leq e^\top A e \leq \lambda_{\max}(A) \|e\|_2^2.
\end{equation}
\end{lemma}

\begin{proof}
Here, we only prove the part concerning the maximum value. The proof for the minimum value is completely analogous.

Since $ A $ is a real symmetric matrix, $ A $ has an orthonormal basis of eigenvectors $ e_1, \ldots, e_n $ in $ \mathbb{R}^n $. This means that $ A e_i = \lambda_i(A) e_i $, where $ \lambda_i(A) $ is the $ i $-th eigenvalue of $ A $, and $ e_j^\top e_i = \begin{cases} 0, & i \neq j \\ 1, & i = j \end{cases} $. 

Thus, for any vector $ e = \sum_{i=1}^n k_i e_i $ in $ \mathbb{R}^n $, we have:

$$
\begin{aligned}
e^\top A e &= \left( \sum_{j=1}^n k_j e_j^\top \right) A \left( \sum_{i=1}^n k_i e_i \right) \\
&= \left( \sum_{j=1}^n k_j e_j^\top \right) \left( \sum_{i=1}^n \lambda_i(A) k_i e_i \right) \\
&= \sum_{1 \leq i, j \leq n} \lambda_i(A) k_i k_j e_j^\top e_i \\
&= \sum_{i=1}^n \lambda_i(A) k_i^2 \quad \text{(since $ e_j^\top e_i = 0 $ for $ i \neq j $)} \\
&= \sum_{i=1}^n \lambda_{\max}(A) k_i^2 + \sum_{i=1}^n (\lambda_i(A) - \lambda_{\max}(A)) k_i^2 \\
&\leq \lambda_{\max}(A) \sum_{i=1}^n k_i^2 \quad \text{(since $ \lambda_i(A) - \lambda_{\max}(A) \leq 0 $)} \\
&= \lambda_{\max}(A) e^\top e.
\end{aligned}
$$

This completes the proof for the maximum value. The argument for the minimum value follows similarly by replacing $ \lambda_{\max}(A)$ with $ \lambda_{\min}(A)$.
\end{proof}

\subsection{Proof of Lemma~\ref{lem:min_value}}
\label{appendix:proof_min_value}

\begin{lemma}
Define $s=\frac{r}{r-1}$, $K(x)=ax^{r}-bx$, where $x, a, b \in \mathbb{R}_{>0}$, and $r>1$. The following conclusions are established:
\begin{itemize}
    \item When $x=\left(\frac{b}{ra}\right)^{s-1}$, $K(x)$ attains its minimum value, given by
$\min_x K(x)=-s^{-1}b^{s}(ra)^{1-s}$.
    \item When $0<x< \left(\frac{b}{a}\right)^{s-1}$, it holds that $K(x)\le 0$.
\end{itemize}
\end{lemma}

\begin{proof}
Given that $K(x)=ax^{r}-bx$, we can derive
\begin{equation}
    \nabla_x K(x)=arx^{r-1} - b.
\end{equation}
Since $a\neq 0$, then $\nabla_x K(x)$ takes 0 when $x=\left(\frac{b}{ra}\right)^{1/(r-1)}$.
Since $K(x)$ is convex with respect to $x$, substituting $x=\left(\frac{b}{ra}\right)^{1/(r-1)}$ into $K(x)$ yields the minimum value of $K(x)$ as follows:
\begin{equation}
    \begin{aligned}
       \min_a K(x)&=\left(\frac{b}{ra}\right)^{r/(r-1)}a-\left(\frac{b}{ra}\right)^{1/(r-1)}b\\
        &=b^{r/(r-1)}(ra)^{-1/(r-1)}r^{-1}-b^{r/(r-1)}(ra)^{-1/(r-1)}\\
        &=-s^{-1}b^{s}(ra)^{1-s},
    \end{aligned}
\end{equation}
where $s=\frac{r}{r-1}$.
When $x=0$ or $x=\left(\frac{b}{a}\right)^{s-1}$, it holds that $K(x)=0$.
Therefore, when $0<x< \left(\frac{b}{a}\right)^{s-1}$, it holds that $K(x)\le 0$.

\end{proof}

\subsection{Proof of Lemma~\ref{lem:prop_thomo_fun}}
\label{appendix:proof_prop_thomo_fun}
\begin{proposition}
Let $\Phi:\mathbb{R}^m\to \bar{\mathbb{R}} \in \mathcal{H}(r_{\Phi})$. Then:
\begin{enumerate}
    \item  The Legendre-Fenchel conjugate $\Phi^*$ satisfies $\Phi^* \in \mathcal{H}(r_{\Phi^*})$, with $1/r_{\Phi} + 1/r_{\Phi^*} = 1$. 
    
    \item The relationship between $\Phi$ and its conjugate is given by:
    \begin{equation}
        r_\Phi \Phi(\mu) = r_{\Phi^*} \Phi^*(\mu^*_\Phi) = \mu^\top \mu_\Phi^*,
    \end{equation}
    where $\mu^*_\Phi$ is the dual variable associated with $\mu$.
    
    \item The normalized norm power function satisfies the triangle inequality:
    \begin{equation}
        \bar{\Phi}(\mu + \nu) \leq \bar{\Phi}(\mu) + \bar{\Phi}(\nu).
    \end{equation}
    
    \item  The product of the normalized norm power function and its conjugate satisfies:
    \begin{equation}
        \bar{\Phi}(\mu) \bar{\Phi}^*(\nu) \geq |\langle \mu, \nu \rangle|,
    \end{equation}
    where $\langle \cdot, \cdot \rangle$ denotes the inner product. 
    
    \item The equivalence between the $L_2$ norm and $\bar{\Phi}$ is established as:
    \begin{equation}
    \begin{aligned}
    m^{-1/2} \left(\sum_{i=1}^m \bar{\Phi}^*(\mathbf{e}_i)^2\right)^{-1/2} \|\mu\|_2 
    &\leq \bar{\Phi}(\mu) \\
    &\leq \left(\sum_{i=1}^m \bar{\Phi}(\mathbf{e}_i)^2\right)^{1/2} \|\mu\|_2,
    \end{aligned}
    \end{equation}
    where $\mathbf{e}_i$ denotes a standard basis vector (``one-hot'' vector).
\end{enumerate}
\end{proposition}
\begin{proof}

For $\Phi$ is a norm power function with order $r_\Phi$, that means $\forall k>0$, 
\begin{equation}
    \Phi(k\mu)=k^{r_{\Phi}}\Phi(\mu).
\end{equation}
Letting $\Omega(\mu)=\Phi(k\mu)$ instead, according to the definition of Legendre-Fenchel conjugate, we have 
\begin{equation}
\begin{aligned}
\Omega^*(\nu) &= \sup_\mu \, \langle \nu, \mu \rangle - \Omega(\mu) \\ 
&= \sup_\mu \, \langle \nu, \mu \rangle - \Phi(k\mu) \\
 &= \sup_{q} \, \langle \nu/k, q \rangle - \Phi(q) \\
&= \Phi^*(\nu/k),
\end{aligned}
\end{equation}
where, we made a change of variable $q=k\mu$.

Letting $\Omega(\mu)=k^{r_{\Phi}}\Phi(\mu)$, then according to the definition of Legendre-Fenchel conjugate, we have  
$$
\begin{aligned}
\Omega^*(\nu) &= \sup_\mu \, \langle \nu, \mu \rangle - \Omega(\mu) \\
&= \sup_\mu \, \langle \nu, \mu \rangle - k^{r_{\Phi}} \Phi(\mu) \\
&= k^{r_{\Phi}}  \sup_\mu \, \langle \nu/k^{r_{\Phi}}, \mu \rangle - \Phi(\mu) \\
&= k^{r_{\Phi}}  \Phi^*(\nu/k^{r_{\Phi}}).
\end{aligned}
$$
Given that $\Phi(k\mu) = k^{r_{\Phi}}\Phi(\mu)$, it follows that their Legendre-Fenchel conjugates are also equal, namely
\begin{equation}
    \Phi^*(\nu/k) = k^{r_{\Phi}} \Phi^*(\nu/k^{r_{\Phi}}).
\end{equation}
Letting $g=\nu/k$, the above equality becomes 
\begin{equation}\label{eq:h_c_conjugate}
    \Phi^*(g)=k^{-{r_{\Phi}}}\Phi^*(k^{1-{r_{\Phi}}}g).
\end{equation}
Letting $\lambda = k^{1-{r_{\Phi}}}$, since ${r_{\Phi}}>1$, it follows that 
\begin{equation}
    k^{-{r_{\Phi}}}=\lambda^{{r_{\Phi}}/({r_{\Phi}}-1)}.
\end{equation}
Substituting the above equality into formula~\eqref{eq:h_c_conjugate}, we obtain
\begin{equation}
    \lambda^{{r_{\Phi}}/({r_{\Phi}}-1)}\Phi^*(g)=\Phi^*(\lambda g).
\end{equation}
Therefore, $\Phi^*$ is a norm power function of order $r_{\Phi^*}=r_{\Phi}/(r_{\Phi}-1)$, that is $1/r_{\Phi}+1/r_{\Phi^*}=1$.

Since $\Phi$ is strictly convex and differentiable, we have
\begin{equation}
    \begin{aligned}
        \Phi^*(\nu)&=\sup_{\mu}\langle \mu,\nu\rangle - \Phi(\mu)\\
                      &\ge -\Phi(0)\\
                      &=0.
    \end{aligned}
\end{equation}
Therefore, $\Phi^*$ is non-negative.
Since $\Phi(0)+\Phi^*(0_{\Phi}^*)=0$ and $\Phi(0)=0$, it means that $\Phi^*$ takes its minimum value of 0 at 0.
According to Lemma~\ref{lem:convex_diff}, since $\Phi$ is a Legendre function, then $\Phi^*$ is also a Legendre function.
Therefore, $\Phi^*$ is a norm power function of order $r_{\Phi^*}=r_{\Phi}/(r_{\Phi}-1)$.
Thus, Conclusion 1 is proven.

Below, we proceed to prove Conclusion 2 
According to Euler's Homogeneous Function Theorem~\ref{lem:euler}, we have
\begin{equation}
    \begin{aligned}
        r_\Phi\Phi(\mu)&=\mu^\top \nabla_{\mu} \Phi(\mu)=\mu^\top \mu_{\Phi}^*,\\
        r_{\Phi^*}\Phi^*(\mu^*_\Phi)&=(\mu_\Phi^*)^\top\nabla_{\mu^*_\Phi} \Phi^*(\mu_\Phi^*)=(\mu_{\Phi}^*)^\top \mu.
    \end{aligned}
\end{equation}
Since $\Phi$ is strictly convex, it follows that $\mu^\top \mu_\Phi^* = \Phi(\mu) + \Phi^*(\mu_\Phi^*)$. Consequently, we have:

\begin{equation}
    \begin{aligned}
        \Phi^*(\mu_\Phi^*) &= (r_\Phi - 1)\Phi(\mu), \\
        \Phi(\mu) &= (r_{\Phi^*} - 1)\Phi^*(\mu_\Phi^*).
    \end{aligned}
\end{equation}

Thus, Conclusion 2 is proven.

We proceed to prove Conclusion 3. 
Given that $\bar{\Phi}(\cdot)$ is a norm, the function $\bar \Phi(\nu)$ is convex and homogeneous of degree $1$. 
We can observe that
\begin{equation}
\begin{aligned}
        \frac{1}{2}\bar \Phi(\mu+\nu)&=\bar \Phi\left(\frac{\mu+\nu}{2}\right)\\
        &\le \frac{1}{2}(\bar \Phi(\mu)+\bar \Phi(\nu)),
\end{aligned}
\end{equation}
which indicates that $\bar \Phi$ satisfies the triangle inequality.
Thus, Conclusion 3 is proven.

We begin to prove Conclusion 4.
Since $\Phi,\Phi^*$ are norm power functions of degree $r_\Phi$ and $r_{\Phi^*}$, respectively, Proposition~\ref{prop:n_loss_duality} implies that $r_{\Phi^*}+r_\Phi=r_{\Phi^*} r_\Phi$.
From the definition of Fenchel-Young loss, it follows that 
\begin{equation}
d_{{\Phi}}(a\mu,\nu/a)=a^{r_\Phi}\Phi(\mu)+a^{-{r_{\Phi^*}}}\Phi^*(\nu)-\langle \mu,\nu\rangle,
\label{opt_1}
\end{equation}
where $a\in \mathbb{R}_{>0}$. 
By taking the first and second derivatives of $d_{{\Phi}}(a\mu,\nu/a)$ separately, we obtain 
\begin{equation}
    \begin{aligned}
        \nabla_a d_{{\Phi}}(a\mu,\nu/a) &= r a^{{r_\Phi}-1}\Phi(\mu)-{r_{\Phi^*}}a^{-{r_{\Phi^*}}-1}\Phi^*(\nu),\\
        \nabla^2_a d_{{\Phi}}(a\mu,\nu/a) &= {r_\Phi}({r_\Phi}-1)a^{{r_\Phi}-2}\Phi(\mu)+{r_{\Phi^*}}({r_{\Phi^*}}+1)a^{-{r_{\Phi^*}}-2}\Phi^*(\nu).
    \end{aligned}
\end{equation}
Since $r_\Phi>1$, the function $d_\Phi(a\mu,\nu/a)$ is convex with respect to $a$. 
Solving the equality $\nabla_a d_{a{\Phi}}(\mu,\nu) = 0$,  we find 
\begin{equation}\label{eq:a_value}
    a=\left(\frac{{r_{\Phi^*}}\Phi^*(\nu)}{{r_\Phi}\Phi(\mu)}\right)^{1/{r_\Phi}{r_{\Phi^*}}}.
\end{equation}
Therefore, $d_{a{\Phi}}(\mu,\nu)$ attains a globally unique minimum value when $a=\left(\frac{{r_{\Phi^*}}\Phi^*(\nu)}{{r_\Phi}\Phi(\mu)}\right)^{1/{r_\Phi}{r_{\Phi^*}}}$. 

By substituting equality~\eqref{eq:a_value} into equality~\eqref{opt_1}, we have 
\begin{equation}
\begin{aligned}
        \min_{a\in \mathbb{R}_{>0}}d_\Phi(a\mu,\nu/a)&= ({r_{\Phi^*}}\Phi^*(\nu))^{1/{r_{\Phi^*}}}({r_\Phi}\Phi(\mu))^{1/{r_\Phi}} -\langle \mu, \nu\rangle \\
        &\ge 0.
\end{aligned}
\end{equation}
Since $\Phi(\mu),\Phi^*(g)$ are even functions, we can then derive
\begin{equation}
\begin{aligned}
        \min_{a\in \mathbb{R}_{>0}}d_\Phi(a\mu,-\nu/a) &= ({r_{\Phi^*}}\Phi^*(\nu))^{1/{r_{\Phi^*}}}({r_\Phi}\Phi(\mu))^{1/{r_\Phi}} -\langle \mu, -\nu\rangle\\
        &\ge 0.
\end{aligned}
\end{equation}
Therefore, we obtain $\bar \Phi^*(\nu) \bar \Phi(\mu) -|\langle \mu, \nu\rangle|\ge 0$.
Thus, Conclusion 4 is proven.

Finally, leveraging Conclusion 3 and Conclusion 4, we proceed to prove Conclusion 5.
Since $\mathbf{e}_i$ denotes a standard basis (``one-hot'') vector, we have
\begin{equation}
    \mu=\sum_{i=1}^m \mu_i \mathbf{e}_i.
\end{equation}
For vector $\mu$, we define $\mathrm{sgn}(\mu)=\sum_{i=1}^m \mathrm{sgn}(\mu_i)\mathbf{e}_i$, where $\mathrm{sgn}$ is the sign function, that is $\mathrm{sgn}(\mu_i)={\mu_i}/{|\mu_i|}$, where $|\mu_i|$ represents the absolute value of $\mu_i$.
Based on Proposition~\ref{prop:triangle}, we have 
\begin{equation}\label{eq:base_1}
    \begin{aligned}
        \bar \Phi(\mu)&=\bar \Phi\left(\sum_{i=1}^m \mu_i\mathbf{e}_i\right)\\
        &\le \sum_{i=1}^m \bar \Phi(\mu_i\mathbf{e}_i)\\
        &= \sum_{i=1}^m |\mu_i|\bar \Phi(\mathbf{e}_i)\\
        &\le \|\mu\|_2 \left(\sum_{i=1}^m \bar \Phi(\mathbf{e}_i)^2\right)^{1/2},
    \end{aligned}
\end{equation}
where the last inequality is obtained from the Cauchy-Schwarz inequality.
Therefore, it follows that
\begin{equation}\label{eq:n_up_b}
\begin{aligned}
        \Phi(\mu)&\le \frac{1}{{r_{\Phi}}}\left(\|\mu\|_2 \left(\sum_{i=1}^m \bar \Phi(\mathbf{e}_i)^2\right)^{1/2}\right)^{r_{\Phi}}\\
    &=\frac{1}{r_{\Phi}}\|\mu\|_2^{r_{\Phi}} \left(\sum_{i=1}^m \bar \Phi(\mathbf{e}_i)^2\right)^{r_{\Phi}/2}.
\end{aligned}
\end{equation}

Based on Proposition~\ref{lem:prop_thomo_fun:4}, we have
\begin{equation}\label{eq:n_bound_1}
    \begin{aligned}
       \bar \Phi(\mu)\bar \Phi^*(\mathrm{sgn}(\mu))&\ge \langle \mu,\mathrm{sgn}(\mu)\rangle\\
        &=\sqrt{\|\mu\|_1^2}=\sqrt{\sum_{i=1}^m\sum_{j=1}^m|\mu_i||\mu_j|}\\
        & \ge \sqrt{\|\mu\|_2^2}\\
        &=\|\mu\|_2.
    \end{aligned}
\end{equation}
Based on inequality~\eqref{eq:base_1}, we have
\begin{equation}
\begin{aligned}
        \bar \Phi^*(\mathrm{sgn}(\mu))&\le \|\mathrm{sgn}(\mu)\|_2 \left(\sum_{i=1}^m \bar \Phi^*(\mathbf{e}_i)^2\right)^{1/2}\\
    &=\sqrt{m} \left(\sum_{i=1}^m \bar \Phi^*(\mathbf{e}_i)^2\right)^{1/2}.
\end{aligned}
\end{equation}
By taking the above inequality into inequality~\eqref{eq:n_bound_1}, we have
\begin{equation}
    \begin{aligned}
        \bar \Phi^*(\mu)m^{1/2} \left(\sum_{i=1}^m \bar \Phi^*(\mathbf{e}_i)^2\right)^{1/2}\ge \|\mu\|_2.
    \end{aligned}
\end{equation}
It follows that 
\begin{equation}\label{eq:n_low_b}
\begin{aligned}
        \Phi(\mu)&\ge \frac{1}{r_{\Phi}}\left(m^{-1/2} \left(\sum_{i=1}^m \bar \Phi^*(\mathbf{e}_i)^2\right)^{-1/2}\right)^{r_{\Phi}}\|\mu\|_2^{r_{\Phi}}\\
        &=\frac{1}{r_{\Phi}}m^{-r_{\Phi}/2} \left(\sum_{i=1}^m \bar \Phi^*(\mathbf{e}_i)^2\right)^{-r_{\Phi}/2}\|\mu\|_2^{r_{\Phi}}.
\end{aligned}
\end{equation}

By combining inequality~\eqref{eq:n_up_b} and inequality~\eqref{eq:n_low_b}, we have
\begin{equation}
\begin{aligned}
       &\frac{1}{r_{\Phi}}m^{-r_{\Phi}/2} \left(\sum_{i=1}^m \bar \Phi^*(\mathbf{e}_i)^2\right)^{-r_{\Phi}/2}\|\mu\|_2^{r_{\Phi}}\le \Phi(\mu)\le \frac{1}{r_{\Phi}} \left(\sum_{i=1}^m \bar \Phi(\mathbf{e}_i)^2\right)^{r_{\Phi}/2}\|\mu\|_2^{r_{\Phi}},\\
       &m^{-1/2} \left(\sum_{i=1}^m \bar \Phi^*(\mathbf{e}_i)^2\right)^{-1/2}\|\mu\|_2\le \bar \Phi(\mu)\le \left(\sum_{i=1}^m \bar \Phi(\mathbf{e}_i)^2\right)^{1/2}\|\mu\|_2.
\end{aligned}
\end{equation}

\end{proof}

\subsection{Proof of Lemma~\ref{lem:fenchel_young_condition}}
\label{appendix:proof_fenchel_young_condition}

\begin{lemma}
Given a function $g: \mathbb{R}^m \to \mathbb{R}$, let $G(\mu) = d_g(\mu, s)$, where $s$ is a constant vector. Then, it holds that:
\begin{equation}
    d_G(\mu,\nu_G^*)=d_g(\mu,\nu^*_g).
\end{equation}
\end{lemma}

\begin{proof}
    Since $G(\mu) = d_g(\mu,s)$. By the definition of the Fenchel-Young loss, we have:
\begin{equation}
d_{G}(\mu, \nu_{G}^*) = G(\mu) + G^*(\nu_G^*) - \langle \mu, \nu_G^* \rangle.
\end{equation}

Expanding this expression step by step:
\begin{equation}
\begin{aligned}
d_{G}(\mu, \nu_G^*) &= G(\mu) + G^*(\nu_{F_y}^*) - \langle \mu, \nu_G^* \rangle \\
&= G(\mu) + G^*(\nu_G^*) - \langle \nu, \nu_G^* \rangle + \langle \nu, \nu_G^* \rangle - \langle \mu, \nu_G^* \rangle \\
&= G(\mu) - G(\nu) - \langle \nu_G^*, (\mu - \nu) \rangle.
\end{aligned}
\end{equation}

Substituting $G(\mu) = d_g(\mu,s)$, we get:
\begin{equation}
\begin{aligned}
d_{G}(\mu, \nu_{G}^*) &= d_g(\mu,s) - d_g(\nu,s) - \langle (\nu_{g}^* - s), (\mu - \nu) \rangle \\
&= g(\mu) - \langle s, \mu - \nu \rangle - g(\nu) - \langle (\nu_{g}^* - s), (\mu - \nu) \rangle \\
&= g(\mu) - g(\nu) - \langle \nu_{g}^*, (\mu - \nu) \rangle \\
&= d_{g}(\mu, \nu_{g}^*)
\end{aligned}
\end{equation}
\end{proof}

\subsection{Proof of Lemma~\ref{lem:h_bound}}
\label{appendix:proof_h_bound}

\begin{theorem}

Let $G_* = \min_{\mu} G(\mu)$, $\mathcal{G}_* = \{\mu \mid G(\mu) = G_*\}$, and $\mu_* \in \mathcal{G}_*$.

\textbf{1. } If $G(\mu)$ is $\mathcal{H}(\Phi,c_\Phi)$-smooth, then it follows that:
\[
\Phi^*(\mu_G^*) - s_{\Phi} \leq G(\mu) - G_* \leq \Phi(\mu - \mu_*) + c_\Phi.
\]

\textbf{2. } If $G(\mu)$ is $\mathcal{H}(\phi,c_\phi)$-convex, then we have:
\[
\phi(\mu - \mu_*) - c_{\phi} \leq G(\mu) - G_* \leq \phi^*(\mu_G^*) + c_\phi.
\]

\textbf{3. } If $G(\mu)$ is both $\mathcal{H}(\Phi,c_\Phi)$-smooth and $\mathcal{H}(\phi,c_\phi)$-convex, we obtain:
\[
\Phi^*(\mu_G^*) - s_{\Phi} \leq G(\mu) - G_* \leq \phi^*(\mu_G^*) + c_\phi.
\]

\end{theorem}
\begin{proof}
According to the definition of Fenchel-Young loss, we have
\begin{equation}\label{eq:def_base}
\begin{aligned}
        G_*&= \min_{\nu} G(\nu)\\
        &\le \min_{\nu} \{G(\mu)+(\mu_G^*)^\top (\nu-\mu) +d_G(\nu,\mu_G^*)\}\\
        &=\min_{\nu} \{G(\mu)+(\mu_G^*)^\top z+d_G(\nu,\mu_G^*)\},
\end{aligned}
\end{equation} where $z=\nu-\mu$.

Since $ \mu_* \in \mathcal{G}_* $ represents a global minimum, which must also be a stationary point, it follows that $ \nabla G(\mu_*) = 0 $.
Since $G$ is $\mathcal{H}(\Phi,c_\Phi)$-smooth, according to the definition of Fenchel-Young loss, we have 
\begin{equation}
    \begin{aligned}
        G(\mu)&\le G(\mu_*)+\nabla G(\mu_*)^\top (\mu-\mu_*) +d_G(\mu,(\mu_*)_G^*)\\
                 &\le G_*+\Phi(\mu-\mu_*)+c_\Phi.
    \end{aligned}
\end{equation}
Thus, the upper bound of the first conclusion is proved.

Since $G$ is $\mathcal{H}(\Phi)$-smooth, we have 
\begin{equation}
    d_G(\nu,\mu_G^*)\le \Phi(z)+c_\Phi,
\end{equation} where $z=\nu-\mu$.
By substituting the above inequality into formula~\eqref{eq:def_base}, it follows that 
\begin{equation}
    G_*\le\min_z \{G(\mu)+ (\mu_G^*)^\top z+\Phi(z)\}+c_\Phi.
\end{equation} 
When $z_{\Phi}^*=-\mu_G^*$, the right side of the inequality takes the minimum value. 

Letting $z_{\Phi}^*=-\mu_G^*$.
Since $\Phi$ and $\Phi^*$ are even and homogeneous functions of degree $r_\Phi$ and $r_{\Phi^*}$, according to Proposition~\ref{prop:homo_eq}, we then have
\begin{equation}
    \begin{aligned}
        r_{\Phi^*}\Phi^*(\mu_G^*)&=r_{\Phi^*}\Phi^*(z_{\Phi}^*)\\
        &=( z_{\Phi}^*)^\top\nabla_{z_{\Phi}^*} \Phi^*(z_{\Phi}^*)\\
        &=(z_{\Phi}^*)^\top z\\
        &=-z^\top \mu_G^*,\\
        r_\Phi\Phi(z)&=z^\top \nabla_z \Phi(z)\\
        &=- (\mu_G^*)^\top z.
    \end{aligned}
\end{equation}

Therefore, when $z_{\Phi}^*=-\mu_G^*$, we have 
\begin{equation}\label{eq_nonconvex_1}
\begin{aligned}
   G_* &\le G(\mu)+ (\mu_G^*)^\top z- (\mu_G^*)^\top z/r_\Phi+c_\Phi\\
    &=G(\mu)+ (\mu_G^*)^\top z/r_{\Phi^*}+c_\Phi\\
    &=G(\mu)-\Phi^*(\mu_G^*)+c_\Phi.
\end{aligned}
\end{equation}
The lower bound of the first conclusion is thus proved.

We now proceed to prove the second conclusion.
Since $ \mu_* \in \mathcal{G}_* $ represents a global minimum, which must also be a stationary point, it follows that $ \nabla G(\mu_*) = 0 $.
Since $G$ is $\mathcal{H}(\phi,c_\phi)$-convex, we have 
\begin{equation}
    \begin{aligned}
        G(\mu)&= G(\mu_*)+\nabla G(\mu_*)^\top(\mu-\mu_*) +d_G(\mu,(\mu_*)_G^*)\\
                 &\ge G_*+\phi(\mu-\mu_*)-c_\phi.
    \end{aligned}
\end{equation}
Thus, the lower bound of the second conclusion is proved.

Since $G$ is $\mathcal{H}(\phi)$-convex, we have 
\begin{equation}
\begin{aligned}
    G(\mu_*)&= G(\mu)+\nabla G(\mu)^\top(\mu_*-\mu) +d_G(\mu_*,\mu_G^*)\\
    &\ge G(\mu)+\nabla G(\mu)^\top(\mu_*-\mu) +\phi(\mu_*-\mu)-c_\phi\\
    &\ge \min_z \{G(\mu)+(\mu_G^*)^\top z+\phi(z)\}-c_\phi,
\end{aligned}
\end{equation}
where $z=\mu_*-\mu$.
When $z_{\phi}^*=-\mu_G^*$, the right side of the inequality attains its minimum value. 
Similarly, according to Euler's Homogeneous Function Theorem~\ref{lem:euler}, when $z_{\phi}^*=-\mu_G^*$, we have 
\begin{equation}\label{eq_nonconvex_2}
    \begin{aligned}
     G_* &\ge G(\mu)+(\mu_G^*)^\top z-(\mu_G^*)^\top z/r_\phi-c_\phi\\
    &=G(\mu)+(\mu_G^*)^\top z/r_{\Phi^*}-c_\phi\\
    &=G(\mu)-\phi^*(\mu_G^*)-c_\phi.
    \end{aligned}
\end{equation}
Thus, the upper bound of the second conclusion is proved.

Given that $G(\mu)$ is $\mathcal{H}(\phi,c_\phi)$-convex and $\mathcal{H}(\Phi,c_\Phi)$-smooth, by combining Equation~\eqref{eq_nonconvex_1} and Equation~\eqref{eq_nonconvex_2}, we obtain 
\begin{equation}
    G(\mu)-\phi^*(\mu_G^*)-c_\phi\le G_*\le G(\mu)-\Phi^*(\mu_G^*)+c_\Phi.
\end{equation}
Therefore, the proof of the theorem is complete.
\end{proof}

\subsection{Proof of Theorem~\ref{thm:st_H_eigenvalue_bound}}
\label{appendix:proof_st_H_eigenvalue_bound}
\begin{theorem}

Let $\mathcal{L}_*(z) = \min_{\theta} \mathcal{L}(\theta, z)$, $\mathcal{L}_*(s) = \min_{\theta} \mathcal{L}(\theta, s)$. Define the constants:
\[
C_\Phi = \frac{1}{r_{\Phi^*}} {m_f}^{-r_{\Phi^*}/2} \Big(\sum_{i=1}^{m_f} \bar{\Phi}(\mathbf{e}_i)^2\Big)^{-r_{\Phi^*}/2}, \quad 
C_\phi = \frac{1}{r_{\Phi^*}} \Big(\sum_{i=1}^{m_f} \bar{\phi}^*(\mathbf{e}_i)^2\Big)^{r_{\Phi^*}/2}.
\]
Individual Sample Bound:
If $\lambda_{\min}(A_x) \neq 0$, we have:
\begin{equation}
\begin{aligned}
C_\Phi \frac{\|\nabla_\theta \ell(f_\theta(x), y)\|_2^{r_{\Phi^*}}}{\lambda_{\max}(A_x)^{r_{\Phi^*}/2}} 
-c_\Phi\leq \mathcal{L}(\theta, z) - \mathcal{L}_*(z) 
\leq C_\phi \frac{\|\nabla_\theta \ell(f_\theta(x), y)\|_2^{r_{\Phi^*}}}{\lambda_{\min}(A_x)^{r_{\Phi^*}/2}}+c_\phi.
\end{aligned}
\end{equation}
Dataset-Wide Bound:
If $\lambda_{\min}(A_s) \neq 0$, we have:
\begin{equation}
\begin{aligned}
\frac{C_\Phi \mathbb{E}_Z \big[\|\nabla_\theta \ell(f_\theta(X), Y)\|_2^{r_{\Phi^*}}\big]}{\lambda_{\max}(A_s)^{r_{\Phi^*}/2}} -c_\Phi
\leq \mathcal{L}(\theta, s) - \mathcal{L}_*(s) 
\leq \frac{C_\phi \mathbb{E}_Z \big[\|\nabla_\theta \ell(f_\theta(X), Y)\|_2^{r_{\Phi^*}}\big]}{\lambda_{\min}(A_s)^{r_{\Phi^*}/2}}+c_\phi.
\end{aligned}
\end{equation}

\end{theorem}
\begin{proof}
According to the given Definition~\ref{def:d_s_com_opt} and by applying the chain rule for differentiation, we have
\begin{equation}
            \|\nabla_\theta G(\theta,z)\|_2^2=\|\nabla_\theta \ell(f_{\theta}(x),y)\|_2^2=e^\top A_xe,
\end{equation}
where $A_x:=\nabla_\theta f_{\theta}(x)^\top \nabla_\theta f_{\theta}(x)$, $e:=\nabla_f \ell(f_{\theta}(x),y)$.
Since $ A_x $ is a positive semidefinite real symmetric matrix, it follows from Lemma~\ref{lem:eigen_bound} that:

\begin{equation}
   \lambda_{\min}(A_x) \|\nabla_{f} \ell(f_{\theta}(x), y)\|_2^2 \leq \|\nabla_\theta \ell(f_{\theta}(x), y)\|_2^2 \leq \lambda_{\max}(A_x) \|\nabla_{f} \ell(f_{\theta}(x), y)\|_2^2.
\end{equation}

Since $\lambda_{\min}(A_x)\neq 0$, it follows that 

\begin{equation}
\begin{aligned}
       \frac{\|\nabla_\theta \ell(f_{\theta}(x),y)\|_2^2}{ \lambda_{\max}(A_x)}\le \|\nabla_{f}\ell(f_{\theta}(x),y)\|_2^2\le\frac{\|\nabla_\theta \ell(f_{\theta}(x),y)\|_2^2}{ \lambda_{\min}(A_x)}.
\end{aligned}
\end{equation}

According to Lemma~\ref{lem:prop_thomo_fun:5}, we have
\begin{equation}
\begin{aligned}
        &\Phi^*(\nabla_{f}\ell(f_{\theta}(x),y)) \ge \frac{1}{r_{\Phi^*}}{m_f}^{-r_{\Phi^*}/2} (\sum_{i=1}^{m_f} \bar \Phi(\mathbf{e}_i)^2)^{-r_{\Phi^*}/2}\|\nabla_{f}\ell(f_{\theta}(x),y)\|_2^{r_{\Phi^*}}\\
    &\phi^*(\nabla_{f}\ell(f_{\theta}(x),y))\le \frac{1}{r_{\Phi^*}} (\sum_{i=1}^{m_f} \bar \phi^*(\mathbf{e}_i)^2)^{r_{\Phi^*}/2}\|\nabla_{f}\ell(f_{\theta}(x),y)\|_2^{r_{\Phi^*}}.
\end{aligned}
\end{equation}
Therefore, we obtain
\begin{equation}\label{eq:uu_ll}
\begin{aligned}
        &\Phi^*(\nabla_{f}\ell(f_{\theta}(x),y)) \ge\frac{ C_\Phi\|\nabla_\theta \ell(f_{\theta}(x),y)\|^{r_{\Phi^*}}_2}{ \lambda_{\max}(A_x)^{r_{\Phi^*}/2}}\\
    &\phi^*(\nabla_{f}\ell(f_{\theta}(x),y))\le \frac{ C_\phi\|\nabla_\theta \ell(f_{\theta}(x),y)\|_2^{r_{\Phi^*}}}{ \lambda_{\min}(A_x)^{r_{\Phi^*}/2}}.
\end{aligned}
\end{equation}
By taking the inequalities~\eqref{eq:uu_ll} into Lemma~\ref{lem:h_bound}, we have
\begin{equation}
\begin{aligned}
&\frac{ C_\Phi\|\nabla_\theta \ell(f_{\theta}(x),y)\|^{r_{\Phi^*}}_2}{ \lambda_{\max}(A_x)^{r_{\Phi^*}/2}}-c_\Phi\le \Phi^*(\nabla_{f}\ell(f_{\theta}(x),y))-c_\Phi\\
&\le \ell(f_{\theta}(x),y)-\min_\theta  \ell(f_{\theta}(x),y)\le \phi^*(\nabla_{f}\ell(f_{\theta}(x),y))+c_\phi\\
&\le \frac{ C_\phi\|\nabla_\theta \ell(f_{\theta}(x),y)\|_2^{r_{\Phi^*}}}{ \lambda_{\min}(A_x)^{r_{\Phi^*}/2}}+c_\phi.
\end{aligned}
\end{equation}
Taking the expectation of the above inequality with respect to $Z$, we have:
\begin{equation}
\begin{aligned}
&\mathbb{E}_Z\frac{ C_\Phi\|\nabla_\theta \ell(f(\theta,X),Y)\|^{r_{\Phi^*}}_2}{ \lambda_{\max}(A_X)^{r_{\Phi^*}/2}}-c_\Phi\le \mathbb{E}_Z \Phi^*(\nabla_{f}\ell(f(\theta,X),Y))-c_\Phi\\
&\le \mathcal{L}(\theta,s)-\mathcal{L}_*(s)\le \mathbb{E}_Z\phi^*(\nabla_{f}\ell(f(\theta,X),Y))+c_\phi\\
&\le \mathbb{E}_Z\frac{ C_\phi\|\nabla_\theta \ell(f(\theta,X),Y)\|_2^{r_{\Phi^*}}}{ \lambda_{\min}(A_X)^{r_{\Phi^*}/2}}+c_\phi.
\end{aligned}
\end{equation}
Since $\lambda_{\min}(A_s) = \min_{x \in s} \lambda_{\min}(A_x)$, $\lambda_{\max}(A_s) = \max_{x \in s} \lambda_{\max}(A_x)$, it follows that

\begin{equation}
\begin{aligned}
\frac{C_\Phi}{ \lambda_{\max}(A_s)^{r_{\Phi^*}/2}}&\mathbb{E}_Z { \|\nabla_\theta \ell(f(\theta,X),Y)\|^{r_{\Phi^*}}_2}-c_\Phi\le  \mathcal{L}(\theta,s)-\mathcal{L}_*(s)\\
&\le \frac{C_\phi}{ \lambda_{\min}(A_s)^{r_{\Phi^*}/2}}\mathbb{E}_Z { \|\nabla_\theta \ell(f(\theta,X),Y)\|_2^{r_{\Phi^*}}}+c_\phi.
\end{aligned}
\end{equation}

\end{proof}

\subsection{Proof of Theorem~\ref{thm:general_with_class_sgd_convergence}}
\label{appendix:proof_general_with_class_sgd_convergence}

\begin{theorem}
Let $n=|s|$, $m=|s_k|$, and $\|x\|_\Omega=r_{\Omega^*}^{-1/r_{\Omega^*}}\frac{(\|x\|_2^2)}{\bar{\Omega}(x)}$.  
Given that $\mathcal{L}(\theta,z) = \ell(f_{\theta}(x),y)$ is $\mathcal{H}(\Omega,c_\Omega)$-smooth with respect to $\theta$, the application of SGD as defined in~\ref{def:basic-sgd} leads to the following conclusions:
\begin{enumerate}
    \item The optimal learning rate is achieved when 
    $$
    \alpha = \left(\frac{\|\nabla_\theta \mathcal{L}(\theta_k,s_k)\|_2^2}{r_\Omega\Omega(\nabla_\theta \mathcal{L}(\theta_k,s_k))}\right)^{r_{\Omega^*}-1},
    $$
    and it follows that
    \begin{equation}
       \mathcal{L}(\theta_k, s_k) - \mathcal{L}(\theta_{k+1}, s_k)  \geq \|\nabla_\theta \mathcal{L}(\theta_k, s_k)\|^{r_{\Omega_*}}_\Omega-c_\Omega.
    \end{equation}

    \item The number of steps required for SGD to achieve 
    $$
    \mathbb{E} \|\nabla_\theta \mathcal{L}(\theta, s_k)\|_2^{r_{\Omega^*}} \leq \varepsilon^{r_{\Omega^*}} + \frac{n-m}{m}M+c_\Omega
    $$
    is $\mathcal{O}(\varepsilon^{-r_{\Omega^*}})$.

    \item There exists a constant $\gamma \in \mathbb{R}_{>0}$ such that 
$$
\frac{1}{\gamma}\|\mu\|_2^{r_{\Phi^*}} \leq \|\mu\|_\Omega^{r_{\Phi^*}}, \quad \forall x \in \mathbb{R}^{m_f}.
$$

When $m = 1$ (i.e., a single sample per batch), the condition
$$
\frac{1}{\gamma}\mathbb{E}_Z \|\nabla_\theta \ell(f_{\theta}(x),y)\|_2^{r_{\Omega^*}} \leq \mathbb{E}_Z \|\nabla_\theta \ell(f_{\theta}(x),y)\|_\Omega^{r_{\Omega^*}} \leq \varepsilon^{r_{\Omega^*}} + (n-1)M+c_\Omega
$$
is achieved in $\mathcal{O}(1/\varepsilon^{r_{\Omega^*}})$ steps. 
\end{enumerate}

\end{theorem}

\begin{proof}
Since for any $s_k$, after updating using $\nabla_\theta \mathcal{L}(\theta_k,s_k)$, it holds that
\begin{equation}\label{eq:tmp_fang}
    \mathcal{L}(\theta_{k+1}, s \setminus s_k) - \mathcal{L}(\theta_k, s \setminus s_k) \leq M.
\end{equation}

For $\mathcal{L}(\theta_{k+1}, s)$, we have
\begin{equation}
    \begin{aligned}
       \mathcal{L}(\theta_{k+1}, s) &- \mathcal{L}(\theta_k, s) = \frac{n-m}{n}\big(\mathcal{L}(\theta_{k+1}, s \setminus s_k) - \mathcal{L}(\theta_k, s \setminus s_k)\big) \\
       &+ \frac{m}{n}\big(\mathcal{L}(\theta_{k+1}, s_k) - \mathcal{L}(\theta_k, s_k)\big) \\
       &\leq \frac{n-m}{n}M + \frac{m}{n}\big(\mathcal{L}(\theta_{k+1}, s_k) - \mathcal{L}(\theta_k, s_k)\big).
    \end{aligned}
\end{equation}
Let $G_z(\theta)=\ell(f_\theta(x),y)$. 
Since for all $z \in s_k$, $G_z(\theta)$ is $\mathcal{H}(\Omega,c_\Omega)$-smooth, then we have
\begin{equation}
    d_{G_z}(\theta_1,(\theta_2)_{G_z}^*)\le \Omega(\theta_1-\theta_2)+c_\Omega.
\end{equation}
Let $G(\theta)=\mathcal{L}(\theta,s_k)$. By applying SGD with a batch size of $m$, we have
\begin{equation}
\begin{aligned}
    \mathcal{L}(\theta_{k+1}, s_k) &- \mathcal{L}(\theta_k, s_k) \\
    &\le -\alpha \nabla_\theta \mathcal{L}(\theta_k, s_k)^\top \nabla_\theta \mathcal{L}(\theta_k,s_k)+d_G(\theta_{k+1},(\theta_k)_G^*)\\
    &\leq -\alpha \nabla_\theta \mathcal{L}(\theta_k, s_k)^\top \nabla_\theta \mathcal{L}(\theta_k,s_k) + \alpha^{r_\Omega} \Omega(\nabla_\theta \mathcal{L}(\theta_k,s_k)) +c_\Omega\\
    &= -\alpha \|\nabla_\theta \mathcal{L}(\theta_k,s_k)\|_2^2 + \alpha^{r_\Omega} \Omega(\nabla_\theta \mathcal{L}(\theta_k,s_k))+c_\Omega.
\end{aligned}
\end{equation}

Let $K(\alpha) = \alpha^{r_\Omega} \Omega(\nabla_\theta \mathcal{L}(\theta_k,s_k))-\alpha \|\nabla_\theta \mathcal{L}(\theta_k,s_k)\|_2^2$.
Based on Lemma \ref{lem:min_value}, the minimum value of $K(\alpha)$ is attained when $\alpha=(\frac{\|\nabla_\theta \mathcal{L}(\theta_k,s_k)\|_2^2}{r_\Omega\Omega(\nabla_\theta \mathcal{L}(\theta_k,s_k))})^{r_{\Omega^*}-1}$, as given below,
\begin{equation}
    \begin{aligned}
       \min_\alpha K(\alpha)&=-\frac{1}{r_{\Omega^*}}(\|\nabla_\theta \mathcal{L}(\theta_k,s_k)\|_2^2)^{r_{\Omega^*}}(r_\Omega \Omega(\nabla_\theta \mathcal{L}(\theta_k,s_k)))^{1-r_{\Omega^*}}\\
       &=-\frac{1}{r_{\Omega^*}}(\|\nabla_\theta \mathcal{L}(\theta_k,s_k)\|_2^2)^{r_{\Omega^*}}(r_\Omega \Omega(\nabla_\theta \mathcal{L}(\theta_k,s_k)))^{-r_{\Omega^*}/r_\Omega}\\
       &=-\frac{1}{r_{\Omega^*}}(\|\nabla_\theta \mathcal{L}(\theta_k,s_k)\|_2^2)^{r_{\Omega^*}}\bar{\Omega}(\nabla_\theta \mathcal{L}(\theta_k,s_k))^{-r_{\Omega^*}}\\
       &=-\frac{(\|\nabla_\theta \mathcal{L}(\theta_k,s_k)\|_2^2)^{r_{\Omega^*}}}{r_{\Omega^*}\bar{\Omega}(\nabla_\theta \mathcal{L}(\theta_k,s_k))^{r_{\Omega^*}}}\\
    \end{aligned}
\end{equation}
Based on Lemma \ref{lem:min_value}, we also have that when  
\[
0 < \alpha < \left(\frac{\|\nabla_\theta \mathcal{L}(\theta_k, s_k)\|_2^2}{\Omega(\nabla_\theta \mathcal{L}(\theta_k, s_k))}\right)^{r_{\Omega^*} - 1},
\]
it holds that \(K(\alpha) \leq 0\).

Thus, we have
\begin{equation}\label{eq:batch_tmp_2}
\begin{aligned}
    \mathcal{L}(\theta_{k+1}, s_k) &- \mathcal{L}(\theta_k, s_k) \leq -\frac{(\|\nabla_\theta \mathcal{L}(\theta_k,s_k)\|_2^2)^{r_{\Omega^*}}}{r_{\Omega^*}\bar{\Omega}(\nabla_\theta \mathcal{L}(\theta_k,s_k))^{r_{\Omega^*}}}+c_\Omega.
\end{aligned}
\end{equation}
Thus, we have:
\begin{equation}
           \mathcal{L}(\theta_k, s_k) - \mathcal{L}(\theta_{k+1}, s_k)  \geq \|\nabla_\theta \mathcal{L}(\theta_k, s_k)\|^{r_{\Omega_*}}_\Omega - c_\Omega.
\end{equation}
The first conclusion is thereby proved.

Substituting above inequality into Equation~\eqref{eq:tmp_fang} and taking the expectation over $s_k$ on both sides, we have
\begin{equation}
    \begin{aligned}
       \mathbb{E}_{s_k}[\mathcal{L}(\theta_{k+1}, s)] &- \mathcal{L}(\theta_k, s) \leq \frac{n-m}{n}M + \frac{m}{n}\mathbb{E}_{s_k}\big(\mathcal{L}(\theta_{k+1}, s_k) - \mathcal{L}(\theta_k, s_k)\big) \\
       &\leq \frac{n-m}{n}M - \frac{m}{n}\mathbb{E}_{s_k}\|\nabla_\theta \mathcal{L}(\theta_k, s_k)\|^{r_{\Omega_*}}_\Omega+\frac{m}{n}c_\Omega.
    \end{aligned}
\end{equation}

By rearranging the terms of the above inequality, we obtain
\begin{equation}
   \mathbb{E}_{s_k}\|\nabla_\theta \mathcal{L}(\theta_k, s_k)\|^{r_{\Omega_*}}_\Omega\leq \frac{n}{m}\mathbb{E}_{s_k}[\mathcal{L}(\theta_k, s) - \mathcal{L}(\theta_{k+1}, s)] + \frac{n-m}{m}M+c_\Omega.
\end{equation}

Let us now apply the aforementioned inequality to each step $s_0, s_1, \ldots, s_{T-1}$ and compute the expected value (denoted by $\mathbb{E}$, under the assumption that $s_i$ and $s_j$ are independent), as follows:

\begin{equation}
   \mathbb{E}\|\nabla_\theta \mathcal{L}(\theta_k, s_k)\|^{r_{\Omega_*}}_\Omega\leq \frac{n}{m}\mathbb{E}[\mathcal{L}(\theta_k, s) - \mathcal{L}(\theta_{k+1}, s)] + \frac{n-m}{m}M+c_\Omega.
\end{equation}
Summing both sides of this inequality for $k=0,\cdots,T-1$, we have
\begin{equation}
    \begin{aligned}
       \min_{k=0,\ldots,T-1} &\mathbb{E}\|\nabla_\theta \mathcal{L}(\theta_k, s_k)\|^{r_{\Omega_*}}_\Omega\leq \frac{1}{T}\sum_{k=0}^{T-1}\mathbb{E}\|\nabla_\theta \mathcal{L}(\theta_k, s_k)\|^{r_{\Omega_*}}_\Omega\\
       &\leq \frac{1}{T} \sum_{k=0}^{T-1} \frac{n}{m}\mathbb{E}[\mathcal{L}(\theta_k, s) - \mathcal{L}(\theta_{k+1}, s)] + \frac{n-m}{m}M +c_\Omega\\
       &\leq \frac{n}{mT}(\mathcal{L}(\theta_0, s) - \mathcal{L}_*(s)) + \frac{n-m}{m}M+c_\Omega.
    \end{aligned}
\end{equation}

To ensure that $\mathbb{E} \|\nabla_\theta \mathcal{L}(\theta_k, s_k)\|^{r_{\Omega_*}}_\Omega\leq \varepsilon^{r_{\Omega^*}} + \frac{n-m}{m}M+c_\Omega$, it follows that
\begin{equation}
    \begin{aligned}
    \frac{n}{mT}(\mathcal{L}(\theta_0, s) - \mathcal{L}_*(s)) \leq \varepsilon^{r_{\Omega^*}}.
    \end{aligned}
\end{equation}

From this inequality, we derive
\begin{equation}
    \begin{aligned}
       T &\geq \frac{n(\mathcal{L}(\theta_0, s) - \mathcal{L}_*(s))}{m\varepsilon^{r_{\Omega^*}}} \\
       &= \mathcal{O}(\varepsilon^{-r_{\Omega^*}}).
    \end{aligned}
\end{equation} 
This result indicates that the number of iterations $T$ required to achieve the desired accuracy scales as $\mathcal{O}(\varepsilon^{-r_{\Omega^*}})$.

In particular, when $m = 1$, the condition 
    $$
    \mathbb{E}_Z \|\nabla_\theta \ell(f_{\theta}(x),y)\|_\Omega^{r_{\Omega^*}} \leq \varepsilon^{r_{\Omega^*}} + (n-1)M+c_\Omega
    $$
is achieved in $\mathcal{O}(1/\varepsilon^{r_{\Omega^*}})$ steps. 
    
Since $\|x\|_\Omega=r_{\Omega^*}^{-1/r_{\Omega^*}}\frac{(\|x\|_2^2)}{\bar{\Omega}(x)}$, according to Lemma~\ref{lem:prop_thomo_fun:5}, $\bar{\Omega}(x)$ is equivalent to $\|x\|_2$. Therefore, we can conclude that $\|x\|_\Omega$ is equivalent to $\|x\|_2$. Consequently, the expression $\mathbb{E}_Z \|\nabla_\theta \ell(f(\theta,X),Y)\|_\Omega^{r_{\Omega^*}}$ is equivalent to $\mathbb{E}_Z \|\nabla_\theta \ell(f(\theta,X),Y)\|_2^{r_{\Omega^*}}$. 

Thus, we can conclude that when $m = 1$, the condition
$$
\mathbb{E}_Z \|\nabla_\theta \ell(f(\theta,X),Y)\|_2^{r_{\Omega^*}} \leq \gamma \varepsilon^{r_{\Omega^*}} + \gamma (n-1)M+\gamma c_\Omega
$$
is achieved in $\mathcal{O}(1/\varepsilon^{r_{\Omega^*}})$ steps, where $\|\mu\|_2^{r_{\Phi^*}}\le \gamma \|\mu\|_\Omega^{r_{\Phi^*}}$. 

\end{proof}

\subsection{Proof of Theorem~\ref{thm:number_para}}
\label{appendix:proof_number_para}

\begin{theorem}
If the model satisfies the GIC~\ref{def:grad_indep_con} and $|f_{\theta}(x)|\le |\theta|-1$, then with probability at least $1-O(1 / |f_{\theta}(x)|)$ the following holds:
\begin{equation}
    \begin{aligned}
    L(A_x)\leq U(A_x)&\le -\log \left(1-\frac{2 \log |f_{\theta}(x)|}{|\theta|}\right)^2-\log \epsilon^2,\\
      D(A_x)&=U(A_x)-L(A_x)\le \log \left(Z(|\theta|,|f_{\theta}(x)|)+1\right),
    \end{aligned}
\end{equation}
where $Z(|\theta|,|f_{\theta}(x)|)=\frac{2|y|\sqrt{6 \log |f_{\theta}(x)|}}{\sqrt{|\theta|-1}} \left(1-\frac{2 \log |f_{\theta}(x)|}{|\theta|}\right)^2$, is a decreasing function of $|\theta|$ and an increasing function of $|f_{\theta}(x)|$.
\end{theorem}
\begin{proof}
    Since each column of $\nabla_\theta f_{\theta}(x)$ is uniformly from the ball $\mathcal{B}^{|\theta|}_\epsilon:\{x\in\mathbb R^{|\theta|},\|x\|_2\le \epsilon\}$, according to Lemma~\ref{prop:high_dim_2}, then with probability $1-O(1 / |f_{\theta}(x)|)$ the following holds:
\begin{equation}\label{eq:bound_elememnts}
    \begin{aligned}
            &(A_x)_{ii}=\left\|(\nabla_\theta f(x))_i \right\|^2_2 \geq \left(1-\frac{2 \log |f_{\theta}(x)|}{|\theta|}\right)^2\epsilon^2,\text{ for } i=1, \ldots, |h|,\\
            &|(A_x)_{ij}|=|\langle (\nabla_\theta f(x))_i, (\nabla_\theta f(x))_j\rangle| \leq \frac{\sqrt{6 \log |f_{\theta}(x)|}}{\sqrt{|\theta|-1}} \epsilon^2\quad \text { for all } i, j=1, \ldots,|h|, i \neq j.
    \end{aligned}
\end{equation}
Because $A_x=\nabla_\theta f(x)^\top \nabla_\theta f(x)$ is a symmetric positive semidefinite matrix, then we have $\lambda_{\min}(A_x)\le (A_x)_{ii}\le \lambda_{\max}(A_x)$.
Then the following holds with probability $1-O(1 / |f_{\theta}(x)|)$ at least 
\begin{equation}\label{eq:ev_lower_bound}
\begin{aligned}
\lambda_{\min}(A_x)&\geq \left(1-\frac{2 \log |f_{\theta}(x)|}{|\theta|}\right)^2\epsilon^2,\\
        \lambda_{\max}(A_x)&\geq \left(1-\frac{2 \log |f_{\theta}(x)|}{|\theta|}\right)^2\epsilon^2.
\end{aligned}
\end{equation}

According to Lemma~\ref{lemma:gershgorin}, there exist $k,k'\in  \{1,\cdots,|f_{\theta}(x)|\}$
\begin{equation}
    \begin{aligned}
        \lambda_{\max}(A_x)-(A_x)_{kk}&\le R_k,\\ 
        (A_x)_{k'k'}-\lambda_{\min}(A_x)&\le R_{k'},
    \end{aligned}
\end{equation}
where $R_k=\sum_{1\le j\le n,j\neq k}|(A_x)_{kj}|$, $R_{k'}=\sum_{1\le j\le n,j\neq k'}|(A_x)_{k'j}|$.
Based on the inequalities~\eqref{eq:bound_elememnts}, we obtain 
\begin{equation}\label{eq:bound_R}
    R_k+R_{k'}\le 2(|f_{\theta}(x)|-1)\frac{\sqrt{6 \log |f_{\theta}(x)|}}{\sqrt{|\theta|-1}}\epsilon^2.
\end{equation}

Since $(A_x)_{kk}\le \epsilon^2$ and $|f_{\theta}(x)|\le |\theta|-1$, then we have
\begin{equation}
    \begin{aligned}
        \lambda_{\max}(A_x)-\lambda_{\min}(A_x)&\le R_k+R_{k'}+(A_x)_{kk}-(A_x)_{k'k'}\\
        &\le R_k+R_{k'}+\epsilon^2-(A_x)_{k'k'}\\
        &\le 2(|h|-1)\frac{\sqrt{6 \log |f_{\theta}(x)|}}{\sqrt{|\theta|-1}}\epsilon^2+\epsilon^2-\left(1-\frac{2 \log |f_{\theta}(x)|}{|\theta|}\right)^2\epsilon^2\\
        &\le 2(|f_{\theta}(x)|-1)\epsilon^2\frac{\sqrt{6 \log |f_{\theta}(x)|}}{\sqrt{|\theta|-1}}+\frac{4 \log |f_{\theta}(x)|}{|\theta|}\epsilon^2\\
        &\le \frac{2|f_{\theta}(x)|\epsilon^2\sqrt{6 \log |f_{\theta}(x)|}}{\sqrt{|\theta|-1}}.
    \end{aligned}
\end{equation}
By dividing both sides of the aforementioned inequality by the corresponding sides of inequalities~\eqref{eq:ev_lower_bound}, we obtain:
\begin{equation}
    \begin{aligned}
   \frac{ \lambda_{\max}(A_x)}{\lambda_{\min}(A_x)}&\le Z(|\theta|,|f_{\theta}(x)|)+1,
    \end{aligned}
\end{equation}
where $Z(|\theta|,|f_{\theta}(x)|)=\frac{2|f_{\theta}(x)|\sqrt{6 \log |f_{\theta}(x)|}}{\sqrt{|\theta|-1}} \left(1-\frac{2 \log |f_{\theta}(x)|}{|\theta|}\right)^2$.
\end{proof}

\section{Notation Table}
\label{appendix:notation_table}
For ease of reference, we list the symbol variables defined in Table~\ref{tab:variable_def_abb}.
\begin{table}[htbp]
\caption{Definitions and abbreviations of key variables.}
\label{tab:variable_def_abb}
\begin{small}
	\centering
	\begin{tabularx}{\linewidth}
         {
          >{\hsize=0.5\hsize\linewidth=\hsize\raggedright\arraybackslash}X
          >{\hsize=1.5\hsize\linewidth=\hsize\raggedright\arraybackslash}X
         }
		\toprule
            $\partial f(x)$ & sub-differential of a proper function $f$.
		\\
  		\midrule
            $|f(x)|,|\theta|$ & the length of vector $f(x)$ and parameter number, respectively.
            \\
  		\midrule
            $\lambda_{\max}(A)$, $\lambda_{\min}(A)$ & the maximum and minimum eigenvalues of matrix $A$, respectively. 
            \\
  		\midrule
            $\Omega^*$ & the Legendre-Fenchel conjugate of function $\Omega$. 
            \\
  		\midrule
            $\mu_\Omega^*$ & the gradient of $\Omega$ at $\mu$, i.e., $\mu_{\Omega}^* = \nabla \Omega^*(\mu)$.
            \\
  		\midrule
            $\mathcal{Z},\mathcal{X},\mathcal{Y}$ & the instance space, the input feature space, and the label set, as defined in Subsection~\ref{subsec:basic_setting}.
            \\
  		\midrule
            $\bar{q}$ & the unknown true distribution of the random variable $Z$, with its support on $\mathcal{Z}$.
            \\
  		\midrule
            $s^n(s)$ & the $n$-tuple training dataset, consisting of i.i.d. samples drawn from the unknown true distribution $\bar{q}$, as defined in Subsection~\ref{subsec:basic_setting}.
		\\
  		\midrule
            $q^n$ ($q$) & the empirical PMF of $Z$, derived from these samples $s^n$, as defined in Subsection~\ref{subsec:basic_setting}.
            \\
            \midrule
            $\mathcal{F}_{\Theta}$ & the hypothesis space is a set of functions parameterized by $\Theta$.
            \\
            \midrule
            $f_{\theta}(x)$ ($f(x)$) & the model parameterized by $\theta$.
            \\
            \midrule
            $m_x$, $m_y$, and $m_f$ & the dimensions of the input $x$, the label $y$, and the model output $f_{\theta}(x)$.
            \\
  		\midrule
            $\mathcal{L}(\theta, s)$ & the empirical risk~\eqref{eq:emp_risk}.
            \\
  		\midrule
            $\ell(f_{\theta}(x), y)$ & the loss associated with a single instance $z = (x, y)$.
            \\
  		\midrule
            $A_x$ & the structural matrix corresponding to the model with input $x$. 
            \\
  		\midrule
            $S(\alpha, \beta, \gamma, A_x)$ & the structural error of the model $f_{\theta}(x)$, as defined in Definition~\ref{def:structural_error}.
            \\
  		\midrule
            $U(A_x), L(A_x), D(A_x)$ & $U(A_x) = -\log \lambda_{\min}(A_x)$, $L(A_x) = -\log \lambda_{\max}(A_x)$, $D(A_x) = U(A_x) - L(A_x)$. 
            \\
  		\midrule
            $\alpha$, $\beta$, $\gamma$ & weights associated with $D(A_x)$, $U(A_x)$, and $L(A_x)$, respectively.
            \\
            \midrule
            $\lambda_{\max}(A_s), \lambda_{\min}(A_s)$ & $\lambda_{\min}(A_s) = \min_{x \in s} \lambda_{\min}(A_x)$, $\lambda_{\max}(A_s) = \max_{x \in s} \lambda_{\max}(A_x)$.
            \\
            \midrule
            $\bar{\Phi}$ & normalized version for $\Phi \in \mathcal{H}(r_{\Phi})$.
            \\
            \midrule
            $\mathcal{H}(r_{\Phi})$ & a the class of norm power functions of order $r_\Phi$.
            \\
            \midrule
            $\mathcal{H}(r_{\phi},c_\phi)$-convexity,$\mathcal{H}(r_{\Phi},c_\Phi)$-smoothness & generalized notions of convexity and smoothness; see Definition~\ref{def:g_convex_smooth} for formal definitions.
            \\
            \midrule
            $\mathcal{H}(\phi)$-convexity, $\mathcal{H}(\Phi)$-smoothness &
            $\mathcal{H}(\phi,c_\phi)$-convexity, $\mathcal{H}(\Phi,c_\Phi)$-smoothness hold with no relaxation, i.e., $c_\phi = 0$ and $c_\Phi = 0$; formally defined in Definition~\ref{def:g_convex_smooth}.
            \\
            \midrule
             $s \setminus s_k$ & the set of all elements in $s$ that are not in $s_k$. 
             \\
             \midrule
             $M$ & the gradient correlation factor~\ref{def:gradient_coor_factor}.
             \\
             \midrule 
             $\mathbb{E}_{XY} \|\nabla_\theta \ell(f(\theta,X), Y)\|^{r}_2 $ & the local gradient norm.
              \\
             \midrule
              $ \|\nabla_\theta \mathcal{L}(\theta, s)\|_2^{r_{\Phi^*}}$ & the global gradient norm.
              \\
              \midrule
              $\|x\|_\Omega$ & $\|x\|_\Omega=r_{\Omega^*}^{-1/r_{\Omega^*}}\frac{(\|x\|_2^2)}{\bar{\Omega}(x)}$~\ref{thm:general_with_class_sgd_convergence}.
            \\
		\bottomrule
	\end{tabularx}%

\end{small}
\end{table}

\vskip 0.2in
\bibliography{sample}
\bibliographystyle{theapa}

\end{document}